\documentclass{article}
\usepackage{paper}

\title{Generalization Bounds for Meta-Learning via PAC-Bayes and Uniform Stability}

\author{
  \hspace{0.1in}\textbf{Alec Farid \hspace{2em} Anirudha Majumdar} \\
  Department of Mechanical and Aerospace Engineering, Princeton University\\
  \texttt{\{afarid, ani.majumdar\}@princeton.edu}}

\begin{document}
\maketitle
\vspace{-7pt}
\begin{abstract}
We are motivated by the problem of providing strong generalization guarantees in the context of meta-learning. Existing generalization bounds are either challenging to evaluate or provide vacuous guarantees in even relatively simple settings. We derive a probably approximately correct (PAC) bound for gradient-based meta-learning using two different generalization frameworks in order to deal with the qualitatively different challenges of generalization at the ``base" and ``meta" levels. We employ bounds for uniformly stable algorithms at the base level and bounds from the PAC-Bayes framework at the meta level. The result of this approach is a novel PAC bound that is tighter when the base learner adapts quickly, which is precisely the goal of meta-learning. We show that our bound provides a tighter guarantee than other bounds on a toy non-convex problem on the unit sphere and a text-based classification example. We also present a practical regularization scheme motivated by the bound in settings where the bound is loose and demonstrate improved performance over baseline techniques.
\end{abstract}

\section{Introduction}
\label{introduction}

A major challenge with current machine learning systems is the need to acquire large amounts of training data in order to learn a new task. Over the past few decades, meta-learning \cite{Schmidhuber87, Thrun98} has emerged as a promising avenue for addressing this challenge. Meta-learning relies on the intuition that a new task often bears significant similarity to previous tasks; hence, a learner can learn to perform a new task very quickly by exploiting data from previously-encountered related tasks. The meta-learning problem formulation thus assumes access to datasets from a variety of tasks during meta-training. The goal of the meta learner is then to learn inductive biases from these tasks in order to train a base learner to achieve few-shot generalization on a new task. 

Over the past few years, there has been tremendous progress in practical algorithms for meta-learning (see, e.g., \cite{Santoro16, Ravi17, Finn17, Hospedales20}). Techniques such as model-agnostic meta-learning (MAML) \cite{Finn17} have demonstrated the ability to perform few-shot learning in a variety of supervised learning and reinforcement learning domains. However, our theoretical understanding of these techniques lags significantly behind successes on the empirical front. In particular, the problem of deriving \emph{generalization bounds} for meta-learning techniques remains an outstanding challenge. Current methods for obtaining generalization guarantees for meta-learning \cite{Amit18, Khodak19, Yin20} either (i) produce bounds that are extremely challenging to compute or (ii) produce vacuous or near-vacuous bounds in even highly simplified settings (see Section \ref{sec:examples} for numerical examples). Indeed, we note that existing work on generalization theory for meta-learning techniques do not explicitly report numerical values for generalization bounds. This is in contrast to the state of generalization theory in the supervised learning setting, where recent techniques demonstrate the ability to obtain non-vacuous generalization guarantees on benchmark problems (e.g. visual classification problems \cite{Dziugiate17, Zhou19, Perez-Ortiz20}).

The generalization challenge in meta-learning is similar to, but distinct from, the supervised learning case. In particular, any generalization bound for meta-learning must account for \emph{two levels} of generalization. First, one must account for generalization at the base level, i.e., the ability of the base learner to perform well on new data from a given task. This is particularly important in the few-shot learning setting. Second, one must account for generalization at the meta level, i.e., the ability of the meta learner to generalize to new tasks not encountered during meta-training. Moreover, the generalization performance at the two levels is coupled since the meta learner is responsible for learning inductive biases that the base learner can exploit for future tasks.

The key technical insight of this work is to bound the generalization error at the two levels (base and meta) using two \emph{different} generalization theory frameworks that each are particularly well-suited for addressing the specific challenges of generalization. At the base level, we utilize the fact that a learning algorithm that exhibits uniform stability \cite{Bousquet02, Bousquet20} also generalizes well in expectation (see Section \ref{sec:two_bounds} for a formal statement). Intuitively, uniform stability quantifies the sensitivity of the output of a learning algorithm to changes in the training dataset. As demonstrated by \cite{Hardt16}, limiting the number of training epochs of a gradient-based learning algorithm leads to uniform stability. In other words, a gradient-based algorithm that \emph{learns quickly} is stable. Since the goal of meta-learning is \emph{precisely} to train the base learner to learn quickly, we posit that generalization bounds based on stability are particularly well-suited to bounding the generalization error at the base level. At the meta level, we employ a generalization bound based on \emph{Probably Approximately Correct (PAC)-Bayes} theory. Originally developed two decades ago \cite{McAllester99, Langford03}, there has been a recent resurgence of interest in PAC-Bayes due to its ability to provide strong generalization guarantees for neural networks \cite{Dziugiate17, Arora18, Perez-Ortiz20}. Intuitively, the challenge of generalization at the meta level (i.e., generalizing to new tasks) is similar to the challenge of generalizing to new data in the standard supervised learning setting. In both cases, one must prevent over-fitting to the particular tasks/data that have been seen during meta-training/training. Thus, the strong empirical performance of PAC-Bayes theory in supervised learning problems makes it a promising candidate for bounding the generalization error at the meta level.

\textbf{Contributions.} The primary contributions of this work are the following. First, we leverage the insights above in order to develop a novel generalization bound for gradient-based meta-learning using uniform stability and PAC-Bayes theory (Theorem \ref{thm:meta learning bound}). Second, we develop a regularization scheme for MAML \cite{Finn17} that explicitly minimizes the derived bound (Algorithm \ref{alg:meta-learning with bound}). We refer to the resulting approach as \emph{PAC-BUS} since it combines PAC-Bayes and Uniform Stability to derive generalization guarantees for meta-learning. Third, we demonstrate our approach on two meta-learning problems: (i) a toy non-convex classification problem on the unit-ball (Section \ref{example_circleclass}), and (ii) the \emph{Mini-Wiki} benchmark introduced in \cite{Khodak19} (Section \ref{example_miniwiki}). Even in these relatively small-scale settings, we demonstrate that recently-developed generalization frameworks for meta-learning provide either near-vacuous or loose bounds, while PAC-BUS provides significantly stronger bounds. Fourth, we demonstrate our approach in larger-scale settings where it remains challenging to obtain non-vacuous bounds (for our approach as well as others). Here, we propose a practical regularization scheme which re-weights the terms in the rigorously-derived PAC-BUS upper bound (\emph{PAC-BUS(H)}; Algorithm \ref{alg:meta-learning with bound (H)} in the appendix). Recent work \cite{Yin20} introduces a challenging variant of the \emph{Omniglot} benchmark \cite{Lake11} which highlights and tackles challenges with \emph{memorization} in meta-learning. We show that PAC-BUS(H) is able to prevent memorization on this variant (Section \ref{example_omniglot}).
\vspace{-7pt}
\section{Problem formulation}
\label{problem_formulation}
\vspace{-5pt}

\textbf{Samples, tasks, and datasets.} Formally, consider the setting where we have an unknown meta distribution $\T$ over tasks (roughly, ``tasks" correspond to different, but potentially related, learning problems). A sampled task $t \sim \T$ induces an (unknown) distribution $\Zt$ over sample space $\Z$. We assume that all sampling is independent and identically distributed (i.i.d.). Note that the sample space $\Z$ is shared between tasks, but the distribution $\Zt$ may be different. We then sample within-task samples $z \sim \Zt$ and within-task datasets $S = \{z_1, z_2, \dots, z_m\} \sim \Zt^m$. We assume that each sample $z$ has a single corresponding label $\oracle(z)$, where the function $\oracle$ is an oracle which outputs the correct label of $z$. At meta-training time, we assume access to $l$ datasets, which we call $\textbf{S} = \{S_1, S_2, \dots, S_l\}$. Each dataset $S_i$ in $\textbf{S}$ is drawn by first selecting a task $t_i$ from $\T$, and then drawing $S_i \sim \Zti^m$.

\textbf{Hypotheses and losses.} Let $h$ denote a hypothesis and $L(h, z)$ be the loss incurred by hypothesis $h$ on sample $z$. The loss is computed by comparing $h(z)$ with the true label $\oracle(z)$. For simplicity, we assume that there is no noise on the labels; we can thus assume that all loss functions have access to the label oracle function $\oracle$ and thus the loss depends only on hypothesis $h$ and sample $z$. We note that this assumption is not required for our analysis and is made for the ease of exposition. Overloading the notation, we let $L(h, \Zt) := \EE_{z\sim \Zt}L(h,z)$ and $\widehat{L}(h, S) := \frac{1}{|S|}\sum_{i=1}^{|S|}L(h, z_i)$. 

\textbf{Meta-learning.} As with model-agnostic meta-learning (MAML) \cite{Finn17}, we let meta parameters $\theta \in \thetaspace$ correspond to an initialization of the base learner's hypothesis. Let $h_\theta$ be the $\theta$-initialized hypothesis. Generally, the initialization $\theta$ is learned from the multiple datasets we have access to at meta-training time. In this work, we will learn a \emph{distribution} $\Thetaq$ over initializations so that we can use bounds from the PAC-Bayes framework. At test time, a new task $t \sim \T$ is sampled and we are provided with a new dataset $S \sim \Zt^m$. The base learner uses an algorithm $A$ (e.g., gradient descent), the dataset $S$, and the initialization $\theta \sim \Thetaq$ in order to fine-tune the hypothesis and perform well on future samples drawn from $\Zt$. We denote the base learner's updated hypothesis by $h_\thp$. More formally, our goal is to learn a distribution $\Thetaq$ with the following objective:
\begin{equation}
\label{eq:meta-learning objective}
\min_{\Thetaq} \ \LTT:= \min_{\Thetaq} \ \EEt \ \EESZ \ \EEtheta \ L(h_{\thp},\Zt).
\end{equation}
We are particularly interested in the the few-shot learning case, where the number of samples which the base learner can use to adapt is small. A common technique to improve test performance in the few-shot learning case is to allow for validation data at meta-training time. Thus, in addition to a generalization guarantee on meta-learning without validation data, we will derive a bound when allowing for the use of validation data $S_\va \sim \Zt^n$ during meta-training.

\vspace{-7pt}
\section{Related work}
\label{related_work}
\vspace{-5pt}

\textbf{Meta-learning.} Meta-learning is a well-studied technique for exploiting similarities between learning tasks \cite{Schmidhuber87, Thrun98}. Often used to reduce the need for large amounts of training data, a number of approaches for meta-learning have been explored over decades \cite{Bengio92, Bottou92, Caruana97, Heskes98, Vilalta02, Santoro16, Ravi17, Hospedales20}. Recently, methods based on model-agnostic meta-learning (MAML) \cite{Finn17} have demonstrated strong performance across different application domains and benchmarks such as \textit{Omniglot} \cite{Lake11} and \textit{Mini-ImageNet} \cite{Vinyals16}. These methods operate by optimizing a set of initial parameters that can be quickly fine-tuned via gradient descent on a new task. The approaches mentioned above typically do not provide any generalization guarantees, and none of them compute explicit numerical bounds on generalization performance. Our approach has the structure of gradient-based meta-learning while providing guarantees on generalization.

\textbf{Generalization bounds for supervised learning.} Multiple frameworks have been developed for providing generalization guarantees in the classical supervised learning setting. Early breakthroughs include Vapnik-Chervonenkis (VC) theory \cite{Vapnik68, Anthony99}, Rademacher complexity \cite{Shalev14}, and the minimum description length principle \cite{Blumer87, Rissanen89, Langford05}. More recent frameworks include algorithmic stability bounds \cite{Bousquet02, Celisse16, Hardt16, Rivasplata18, Abou-Moustafa19} and PAC-Bayes theory \cite{Shawe-Taylor97, McAllester99, Seeger02}. The connection between stability and learnability has been established in \cite{Shalev10, Villa13, Hardt16}, and suggests that algorithmic stability bounds are a strong choice of generalization framework. PAC-Bayes theory in particular provides some of the tightest known generalization bounds for classical supervised learning approaches such as support vector machines \cite{Seeger02, Langford03, Germain09, Rivasplata18, Ambroladze06, Parrado-Hernandez07}. Since its development, researchers have continued to tighten \cite{Langford03, McAllester13, Perez-Ortiz20} and generalize the framework \cite{Catoni04, Catoni07, Rivasplata20}. Exciting recent results \cite{Dziugiate17, Neyshabur17, Neyshabur17a, Bartlett17, Arora18, Perez-Ortiz20} have demonstrated the promise of PAC-Bayes to provide strong generalization bounds for neural networks on supervised learning problems (see \cite{Jiang20} for a recent review of generalization bounds for neural networks). It is also possible to combine frameworks such as PAC-Bayes and uniform stability to derive bounds for supervised learning \cite{London17}. We will use these two frameworks to bound generalization in the two levels of meta-learning. In contrast to the standard supervised learning setting, generalization bounds for meta-learning are less common and remain loose. 

\textbf{Generalization bounds for meta-learning.} As described in Section \ref{introduction}, meta-learning bounds must account for two ``levels" of generalization (base level and meta level). The approach presented in \cite{Maurer05} utilizes algorithmic stability bounds at both levels. However, this requires both meta and base learners to be uniformly stable. This is a strong requirement that is challenging to ensure at the meta level. Another recent method, known as follow-the-meta-regularized-leader (FMRL) \cite{Khodak19}, provides guarantees for a regularized meta-learning version of the follow-the-leader (FTL) method for online learning, see e.g. \cite{Hazan16}. The generalization bounds provided are derived from the application of online-to-batch techniques \cite{Alquier17, Denevi19}. A regret bound for meta-learning using an aggregation technique at the meta-level and an algorithm with a uniform generalization bound at the base level is provided in \cite{Alquier17}. The techniques mentioned do not present an algorithm which makes use of validation data (in contrast to our approach). Using validation data (i.e., held-out data) is a common technique for improving performance in meta-learning and is particularly important for the few-shot learning case.

Another method for deriving a generalization bound on meta-learning is to use PAC-Bayes bounds at both the base and meta levels \cite{Pentina14, Pentina15}. In \cite{Amit18}, generalization bounds based on such a framework are provided along with practical optimization techniques. However, the method requires one to maintain distributions over distributions of initializations, which can result in large computation times during training and makes it extremely challenging to numerically compute the bound. Moreover, the approach also does not allow one to incorporate validation data to improve the bound. Recent work has made progress on some of these challenges. In \cite{Rothfuss20}, the computational efficiency of training is improved but the challenges associated with numerically computing the generalization bound or incorporating validation data are not addressed. State-of-the-art work tightens the two-level PAC-Bayes guarantee, addresses computation times for training and evaluation of the bound, and allows for validation data \cite{Yin20}. However, all of the two-level PAC-Bayes bounds require a separate PAC-Bayes bound for each task, and thus a potentially loose union bound.

We present a framework which, to our knowledge, is the first to combine algorithmic stability and PAC-Bayes bounds (at the base- and meta- levels respectively) in order to derive a meta-learning algorithm with associated generalization guarantees. As outlined in Section \ref{introduction}, we believe that the algorithmic stability and PAC-Bayes frameworks are particularly well-suited to tackling the specific challenges of generalization at the different levels. We also highlight that \emph{none} of the approaches mentioned above report numerical values for generalization bounds, even for relatively simple problems. Here, we empirically demonstrate that prior approaches tend to provide either near-vacuous or loose bounds even in relatively small-scale settings while our proposed method provides significantly stronger bounds.

\vspace{-7pt}
\section{Generalization bound on meta-learning}
\label{generalization_bound_for_meta-learning}
\vspace{-5pt}

We use two different frameworks for the two levels of generalization required in a meta-learning bound. We utilize the PAC-Bayes framework to bound the expected training loss on future tasks, and uniform stability bounds to argue that if we have a low training loss when using a uniformly stable algorithm, then we achieve a low test loss. The following section will introduce these frameworks independently. We then present the overall meta-learning bound and associated algorithm to find a distribution over initialization parameters (i.e., meta parameters) that minimizes the upper bound.

\subsection{Preliminaries: two generalization frameworks}
\label{sec:two_bounds}
\subsubsection{Uniform stability}
Let $S = \{z_1, z_2, \dots, z_m\} \in \Z^m$ be a set of $m$ elements of $\Z$. Let $S^i = \{z_1, \dots, z_{i-1}, z_i', z_{i+1},$ $\dots, z_m\}$ be identical to dataset $S$ except that the $i^{th}$ sample $z_i$ is replaced by some $z_i' \in \Z$. Note that our analysis can be extended to allow for losses bounded by some finite $M$, but we work with losses bounded within $[0,1]$ for the sake of simplicity. With these precursors, we define an analogous notion of \emph{uniform stability} to \cite[Definition 2.1]{Hardt16} for deterministic algorithms $A$ and distributions $\Thetaq$ over initializations.\footnote{We use deterministic algorithms to avoid excess computation when calculating the provided meta-learning upper bounds. See Appendix \ref{ap:beta_bounds} for further details.}

\begin{definition}[Uniform Stability] A deterministic algorithm A has $\beta > 0$ uniform stability with respect to loss $L$ if $\ \forall \ z \in \Z, \ \forall \ S \in \Z^m, \ \forall \ i\in\{1, \dots, m\},$ and all distributions $\Thetaq$ over initializations, the following holds: \label{thm:bousquet}
\begin{align}
	\EEtheta \ |L(h_{A(\theta, S)},z) - L(h_{A(\theta, S^i)},z)| \leq \beta.
\end{align}
We define $\b$ as the minimal such $\beta$. 
\end{definition}
In this work, we will bound $\b$ as a function of the algorithm, form of the loss, and number of samples that the algorithm uses (See Appendix \ref{ap:beta_bounds} for further details on the bounds on $\b$ for our setup). We then establish a relationship between uniform stability and generalization in expectation. The following is adapted from \cite[Theorem 2.2]{Hardt16} for the notion of uniform stability presented in Definition \ref{thm:bousquet}:
\begin{theorem}[Algorithmic Stability Generalization in Expectation] Fix a task $t \sim \T$. The following inequality holds for hypothesis $h_\thp$ learned using $\b$ uniformly stable algorithm $A$ with respect to loss $L$:
\label{thm:alg stab gen}
\begin{align}
\label{maurerineq_pre}
    \EESZ \ \EEtheta \ L\hZ \leq \EESZ \ \EEtheta \ \widehat{L}\hS + \b.
\end{align}
\end{theorem}
\begin{proof}
The proof is similar to the one presented for \cite[Theorem 2.2]{Hardt16} and is presented in Appendix \ref{sec:proofthm1}.
\end{proof}

% \alec{We then take the expectation over $t \sim \T$ so that the bound is on $\LTT$ from Equation \eqref{eq:meta-learning objective}:
% \begin{align} 
%     \LTT & = \EEt \ \EESZ \ \EEtheta \ L\hZ \leq  \EEt \ \EESZ \ \EEtheta \ \widehat{L}\hS + \b. \label{maurerineq}
% \end{align}}

\subsubsection{PAC-Bayes theory}
\label{PAC-Bayes theory}
For the meta-level bound, we make use of the PAC-Bayes generalization bound introduced in \cite{McAllester99}. Note that other PAC-Bayes bounds such as the quadratic variant \cite{Rivasplata19} and PAC-Bayes-$\lambda$ variant \cite{Thiemann17} may be used and substituted in the following analysis. We first present a general version of the PAC-Bayes bound and then specialize it to our meta-learning setting in Section \ref{sec:meta-learning bound}. Let $f(\theta, s)$ be an arbitrary loss function which only depends on parameters $\theta$ and the sample $s$ which has been drawn from an arbitrary distribution $\D_s$. The following bound is a tightened version of the bound presented in \cite{McAllester99} for when $l \geq 8$. 
\begin{theorem}[PAC-Bayes Generalization Bound \cite{Maurer04}]
\label{thm:gen pac bayes}
For any data-independent prior distribution $\Thetap$ over $\theta$, some loss function $f$ where $0 \leq f(\theta,s) \leq 1, \forall \ s, \forall \ \theta$, $l \geq 8$, and $\delta \in (0,1)$, with probability at least $1 - \delta$ over a sampling of $\{s_1, s_2, \dots, s_l\} \sim \D_s^l$, the following holds simultaneously for all distributions $\Thetaq$ over $\theta$:
\begin{align}
    \underset{s \sim \D_s}{\EE} \ \EEtheta \ f(\theta,s) \leq & \ \frac{1}{l} \sum_{i=1}^l \ \EEtheta \ f(\theta,s_i) + \Rbayes, \label{gen pac bayes pre}
\end{align} 
where the PAC-Bayes ``regularizer" term is defined as follows 
\begin{align}
    \Rbayes := \Rbayeseq,
\end{align}
and $\KL$ is the Kullback-Leibler (KL) divergence.
\end{theorem}

% \alec{We aim to use Theorem \ref{thm:gen pac bayes} to bound the expected training loss on a dataset $S$ after the base learner is updated. Thus, we define $f(\theta, S) := \widehat{L}\hS$. We would also like $S \sim \D$ to be equivalent to first sampling $t \sim \T$ and then sampling $S \sim \Zt^m$. This is precisely the \emph{marginal distribution} (which we will refer to as $\D_S$) over datasets $S$ of size $m$. In other words, we choose $\D$ to be $\D_S$, where $\D_S$ is the \emph{distribution over datasets} induced by sampling $t \sim \T$ and then $S \sim \Zt^m$. See Appendix \ref{sec:proofthm3} for a formal definition of the marginal distribution $\D_S$ and the stated sampling equivalence.
% For a sampling of the meta-training data $\textbf{S} = \{S_1, S_2, \dots, S_l\} \sim \D_S^l$, Inequality \eqref{gen pac bayes pre} then becomes 
% \begin{align}
% 	 \EES \ \EEtheta \ \widehat{L}\hS = & \ \EEt \ \EESZ \ \EEtheta \ \widehat{L}\hS \nonumber 
% 	\\ \leq & \ \frac{1}{l} \sum_{i=1}^l \ \EEtheta \ \widehat{L}\hiS + \Rbayes,	\label{gen pac bayes}
% \end{align}
% where $A(\theta,S)$ is \emph{any} algorithm which updates $\theta$ using only dataset $S$.}

\subsection{Meta-learning bound}
\label{sec:meta-learning bound}
In order to obtain a generalization guarantee for meta-learning, we utilize the two frameworks above. We first specialize the PAC-Bayes bound in Theorem \ref{thm:gen pac bayes} to bound the expected training loss on future tasks. We then utilize Theorem \ref{thm:alg stab gen} to demonstrate that if we have a low expected training loss when using a uniformly stable algorithm, then we achieve a low expected test loss. These two steps allow us to combine the generalization frameworks above to derive an upper bound on \eqref{eq:meta-learning objective}  which can be computed with known quantities. With the following assumption, the resulting generalization bound is presented in Theorem~\ref{thm:meta learning bound}.

\begin{assumption}[Bounded loss.] The loss function $L$ is bounded: $0 \leq L(h,z) \leq 1$ for any $h$ in the hypothesis space for the given problem, and any $z$ in the sample space.
\label{assumption1}
\end{assumption} 

\begin{theorem}[Meta-Learning Generalization Guarantee]
\label{thm:meta learning bound}
For hypotheses $h_\thp$ learned with $\b$ uniformly stable algorithm A, data-independent prior $\Thetap$ over initializations $\theta$, loss $L$ which satisfies Assumption \ref{assumption1}, $l \geq 8$, and $\delta \in (0,1)$, with probability at least $1 - \delta$ over a sampling of the meta-training dataset $\textbf{S}\sim \D^l_S$, the following holds simultaneously for all distributions $\Thetaq$ over $\theta$:
\begin{align}
    \LTT \leq & \ \frac{1}{l} \sum_{i=1}^l \ \EEtheta \ \widehat{L}(h_\thpi, S_i) + \Rbayes + \b.    \label{meta learning bound}
\end{align}
\end{theorem}
\begin{proof} The proof can be split into three steps:\\
\textbf{Step 1.} \\
Let $\D_s$ in Theorem \ref{thm:gen pac bayes} be the marginal distribution $\D_S$ over datasets of size $m$ (see Appendix \ref{sec:proofthm3} for details) and note that sampling $S \sim \D_S$ is equivalent to first sampling $t\sim \T$ and then sampling $S \sim \Zt^m$. Additionally let $f(\theta, S) := \widehat{L}\hS$ where $A(\theta, S)$ is any deterministic algorithm. Plugging in these definitions into Inequality \eqref{gen pac bayes pre} results in the following inequality which holds under the same assumptions as Theorem \ref{thm:gen pac bayes}, and with probability at least $1-\delta$ over the sampling of  $\textbf{S} \sim \D_S^l$: 
\begin{align}
	 \EES \ \EEtheta \ \widehat{L}\hS = & \ \EEt \ \EESZ \ \EEtheta \ \widehat{L}\hS \nonumber 
	\\ \leq & \ \frac{1}{l} \sum_{i=1}^l \ \EEtheta \ \widehat{L}\hiS + \Rbayes.	\label{gen pac bayes}
\end{align} \\
\textbf{Step 2.} \\
Now assume that algorithm $A$ is $\b$ uniformly stable. For a fixed task $t \sim \T$ we have the following by Theorem \ref{thm:alg stab gen}:
\begin{equation}
    \EESZ \ \EEtheta \ L(h_\thp, \Zt) \leq \EESZ \ \EEtheta \ \widehat{L}\hS + \b.\nonumber
\end{equation}
Take the expectation over $t \sim \T$. We then have:
\begin{equation}\label{maurerineq}
    \EEt \ \EESZ \ \EEtheta \ L(h_\thp, \Zt) \leq \EEt \ \EESZ \ \EEtheta \ \widehat{L}\hS + \b,
\end{equation}
since $\EE_{t \sim \T} \ \b = \b$. This establishes a bound on the true expected loss for a new task after running algorithm A on a training dataset corresponding to the new task. \\
\textbf{Step 3.}\\
% We rewrite \eqref{gen pac bayes pre} using the definition for $f$ and the noted sampling equivalence for $S \sim \D_S$:
% \begin{equation} \tag{\label{gen pac bayes}}
%     \EEt \ \EESZ \ \EEtheta \ \widehat{L}\hS \leq \frac{1}{l} \sum_{i=1}^l \ \EEtheta \ \widehat{L}(h_\thp, S_i) + \Rbayes. 
% \end{equation}
Note that \eqref{gen pac bayes} provides an upper bound on the first term of the RHS of \eqref{maurerineq} when algorithm $A$ is $\b$ uniformly stable. Thus we have the following by plugging \eqref{gen pac bayes} in the RHS of \eqref{maurerineq}: \\
Under the same assumptions as both Theorems \ref{thm:alg stab gen} and \ref{thm:gen pac bayes}, and with probability at least $1-\delta$ over the sampling of $\textbf{S} \sim \D_S^l$:
\begin{equation}
       \EEt \ \EESZ \ \EEtheta \ L(h_\thp, \Zt) \leq \frac{1}{l} \sum_{i=1}^l \ \EEtheta \ L(h_\thp, S_i) + \Rbayes + \b,\nonumber
\end{equation}
completing the proof.
\end{proof}
\vspace{-5pt}

Theorem \ref{thm:meta learning bound} is presented for any distributions $\Thetaq$ and $\Thetap$ over initializations. However, in practice we will use multivariate Gaussian distributions for both. The specialization of Theorem \ref{thm:meta learning bound} to Gaussian distributions is provided in Appendix \ref{ap:bound for gaussians}. Next, we allow for validation data $S_\va \sim \Zt^n$ at meta-training time so that the bound is more suited to the few-shot learning case. We compute the upper bound using the evaluation data $S_\ev = \{S, S_\va\}$ sampled from the marginal distribution $\D_{S_\ev}$ over datasets of size $m+n$. However, we still only require $m$ samples at meta-test time; see Appendix \ref{ap:specialized bound} for the derivation. Note that the training data $S$ is often excluded from the data used to update the meta-learner. However, this is necessary for our approach to obtain a guarantee on few-shot learning performance. The result is a guarantee with high probability over a sampling of $\textbf{S}_\ev \sim \D_{S_\ev}^l$: 
\begin{align}
    \LTT \leq & \ \frac{1}{l} \sum_{i=1}^l \ \EEtheta \ \widehat{L}\hiSev + \Rbayes + \frac{m\b}{m+n}.    
    \label{meta learning bound eval}
\end{align}

\subsection{PAC-BUS algorithm}
Recall that we aim to find a distribution $\Thetaq$ over initializations that minimizes $\LTT$ as stated in Equation \eqref{eq:meta-learning objective}. We cannot minimize $\LTT$ directly due to the expectations taken over unknown distributions $\T$ and $\Zt$ for sampled task $t$, but we may indirectly minimize it by minimizing the upper bounds in Inequalities \eqref{meta learning bound} or \eqref{meta learning bound eval}.

Computing the upper bound requires evaluating an expectation taken over ${\theta \sim \Thetaq}$. In general, this is intractable. However, we aim to minimize this upper bound to provide the tightest guarantee possible. Similar to the method in \cite{Dziugiate17}, we use an unbiased estimator of $\EE_{\theta \sim \Thetaq} L(h_\thp,\cdot)$. Let $\Thetaq$ be a multivariate Gaussian distribution over initializations $\theta$ with mean $\mu$ and covariance $\text{diag}(s)$; thus $\Thetaq = \N(\mu,\text{diag}(s))$ and $\Thetap = \N(\mu_0,\text{diag}(s_0))$. Further, let $\psi := (\mu, \log(s))$, and use the shorthand $\N_{\psi_0}$ for the prior and $\N_\psi$ for the posterior distribution over initializations.
We use the following estimator of $\EE_{\theta \sim \Thetaq} L(h_\thp,\cdot)$: 
\begin{align}
    L(h_\thp, \cdot),\ \ \ \theta \sim \N_\psi.
\end{align}
We present the resulting training technique in Algorithm \ref{alg:meta-learning with bound}. This algorithm can be used to learn a distribution over initializations that minimizes the upper bound presented in Theorem \ref{thm:meta learning bound} and its specializations. This is presented for the case when $A$ is $\b$ uniformly stable for some $\b$. For gradient-based algorithms, the learning rate $\alpha$ often appears directly in the bound for $\b$ \cite{Hardt16}. Thus it is potentially beneficial to update $\alpha$ as well. We present Algorithm  \ref{alg:meta-learning with bound} without learning the learning rate. To meta-learn the learning rate, we can augment $\psi_0$ to include a parameterization of a prior distribution over learning rates and update it using the same gradient step presented in \ref{alg:meta-learning with bound} for $\psi$.

\begin{algorithm}[tb]
    \caption{PAC-BUS: meta-learning via PAC-Bayes and Uniform Stability}
    \label{alg:meta-learning with bound}
\begin{algorithmic}
    \State \textbf{Input}: Fixed prior distribution $\N_{\psi_0}$ over initializations 
    \State \textbf{Input}: $\b$ uniformly stable Algorithm $A$ 
    \State \textbf{Input}: Meta-training dataset $\textbf{S}$, learning rate $\gamma$ 
    \State \textbf{Initialize}: $\psi \leftarrow \psi_0$ 
    \State \textbf{Output}: Optimized $\psi^*$
    \State $B(\psi, \theta'_1,\theta'_2 \dots, \theta'_l) := \frac{1}{l} \sum_{i=1}^l \widehat{L}(h_{\theta'_i},S_i) + \RbayesN +\b$
    \While{not converged}
    \State Sample $\theta \sim \N_\psi$
    \For{$i=1$ {\bfseries to} $l$}
    \State $\theta'_i \leftarrow A(\theta, S_i)$ 
    \EndFor
    \State $\psi \leftarrow \psi - \gamma \nabla_{\psi} B(\psi, \theta'_1,\theta'_2 \dots, \theta'_l)$
    \EndWhile
\end{algorithmic}
\end{algorithm}

Determining the gradient of $B(\psi, \theta'_1,\theta'_2 \dots, \theta'_l)$ with respect to $\psi$ requires computing the Hessian of the loss function if algorithm $A(\theta, S)$ uses a gradient update to compute $\theta'_i$. First order approximations often perform similarly to the second-order meta-learning techniques \cite{Goodfellow15, Finn17, Nichol18}, and can be used to speed up the training. Additionally, Algorithm \ref{alg:meta-learning with bound} can be modified to use mini-batches of tasks instead of all tasks in the meta update to improve training times; we present an algorithm which uses mini-batches of tasks in Appendix \ref{ap:minibatches}.

In practice, we are interested in algorithms such as stochastic gradient descent (SGD) and gradient descent (GD) for the base learner. We can obtain bounds on the uniform stability constant $\b$ when using gradient methods with the results from \cite{Hardt16}. See Appendix \ref{ap:beta_bounds} for details on the $\b$ bounds we use in this work. With a bound on $\b$, we can calculate all the terms in $B(\psi, \theta'_1,\theta'_2 \dots, \theta'_l)$ and use Algorithm \ref{alg:meta-learning with bound} to minimize the meta-learning upper bound. When evaluating the upper bound, we use the sample convergence bound \cite{Langford02, Dziugiate17} to upper bound the expectation taken over $\theta \sim \Thetaq$. See Appendix \ref{ap:sample convergence bound} for details.

\vspace{-7pt}
\section{Examples} 
\label{sec:examples}
\vspace{-7pt}
We demonstrate our approach on three examples below. All examples we provide are few-shot meta-learning problems. To adapt at the base level, $m$ examples from each class are given for an ``$m$-shot" learning problem. If applicable, $n$ samples can be given as validation data for each task during the meta-training step. In the first two examples, our primary goal is to demonstrate the tightness of our generalization bounds compared to other meta-learning bounds. We also present empirical test performance on held-out data; however, we emphasize that the focus of our work is to obtain improved generalization guarantees (and not necessarily to improve empirical test performance). In the third example, we present an algorithm that is motivated by our theoretical framework and demonstrate its ability to improve empirical performance on a challenging task. All the code required to run the following examples is available at \url{https://github.com/irom-lab/PAC-BUS}. %\href{https://github.com/irom-lab/PAC-BUS}{\texttt{https://github.com/irom-lab/PAC-BUS}}. 

\vspace{-7pt}
    \begin{table*}[t]
\caption{We present the generalization bounds (for $\delta = 0.01$) provided by each method if applicable, and use the sample convergence bound \cite{Langford02} for MR-MAML, and PAC-BUS, but not MLAP-M.$^2$ Note that for these methods, we specifically minimize their respective meta-learning bounds. We also report the meta-test loss (the softmax activated cross-entropy loss -- $\CELs$) for all methods. We present the mean and standard deviation after 5 trials. We highlight that our approach provides the strongest generalization guarantee.}
\label{circle_class_results}
% \vskip 0.15in
\begin{center}
\begin{small}
\begin{tabular}{lccccr}
\toprule
 Classification on Ball &\hspace{-0.15in} MAML \cite{Finn17} & \hspace{-0.06in} MLAP-M \cite{Amit18}  &\hspace{-0.1in} MR-MAML \cite{Yin20} & \hspace{-0.1in} PAC-BUS (ours) \\ 
\midrule
  Bound $\downarrow$ &\hspace{-0.15in} None &\hspace{-0.06in} \ $1.0538\pm0.0012^2$ &\hspace{-0.1in} $0.3422\pm0.0006$ & \hspace{-0.13in} \textBF{$0.2213\pm0.0012$} \\ 
 Test Loss $\downarrow$ & \hspace{-0.15in} $0.1701\pm0.0070$ & \hspace{-0.06in}$0.1645\pm0.0045$ &\hspace{-0.13in} \textBF{$0.1584\pm0.0012$} &\hspace{-0.1in} $0.1657\pm0.0014$ \\
\bottomrule
\end{tabular}
\end{small}
\end{center}
\vskip -0.15in
\end{table*} 
\subsection{Example: classification on the unit ball}
\label{example_circleclass}
\vspace{-3pt}

We evaluate the tightness of the generalization bound in Equation \eqref{meta learning bound eval} on a toy two-class classification problem where the sample space $\Z$ is the unit ball $B^2(0,1)$ in two dimensions with radius 1 and centered at the origin. Data points for each task are sampled from $\Zt$, where a task corresponds to a particular concept which labels the data as $(+)$ if within $B^2(c_t,r_t)$ and $(-)$ otherwise. Center $c_t$ is sampled uniformly from the $y \geq 0$ semi-ball $B^2_{y\geq0}(0,0.4)$ of radius $0.4$. The radius $r_t$ is then sampled uniformly from $[0.1, 1-\|c_t\|]$. 
Notably, the decision boundary between classes is nonlinear. Thus, generalization bounds which rely on convex losses (such as \cite{Khodak19}) will have difficulty with providing guarantees for networks that perform well. We choose the softmax-activated cross-entropy loss, $\CELs$, as the loss function. Before running Algorithm \ref{alg:meta-learning with bound}, we address a few technical challenges that arise from Assumption \ref{assumption1} as well as computing $c_L$ and $c_S$. We address these in Appendix \ref{ap:algorithms}. \blfootnote{$^2$Due to high computation times associated with estimating the MLAP upper bound, this value is not computed with the sample convergence bound as the other upper bounds are. Thus, the value presented does not carry a guarantee, but would be similar if computed with the sample convergence bound. The value is shown to give a qualitative sense of the guarantee.}

We then apply Algorithm 1 using the few-shot learning bound in Inequality \eqref{meta learning bound eval}. We present the guarantee on the meta-test loss associated with each training method in Table~\ref{circle_class_results}. In addition, we present the average meta-test loss after training with $10$ samples. We compare our bounds and empirical performance with the meta-learning by adjusting priors (MLAP) technique \cite{Amit18} and the meta-regularized MAML (MR-MAML) technique \cite{Yin20}. All methods are given held-out data to learn a prior before minimizing their respective upper bounds (see Appendix \ref{expdets_circleclass} for further details on the prior training step). Additionally, since all bounds require the loss to be within $[0,1]$, networks $N$ are constrained such that the Frobenius norm of the output is bounded by $r$, i.e., $\|N(z)\|_F \leq r$. We compare the aforementioned methods' meta-test loss to MAML with weights constrained in the same manner (note that MAML does not provide a guarantee). Upper bounds which use the PAC-Bayes framework are computed with many evaluations from the posterior distribution. This allows us to apply the sample convergence bound \cite{Langford02} (as in Equation \eqref{meta learning bound full} for our bound) unless otherwise noted. 

We find that PAC-BUS provides a significantly stronger guarantee compared with the other methods. Note that the guarantee provided by MLAP-M \cite{Amit18} is vacuous because the meta-test loss is bounded between $0$ and $1$, while the guarantee is above $1$.

    \vspace{-5pt}
\subsection{Example: Mini-Wiki}
\label{example_miniwiki}
\vspace{-3pt}

Next, we present results on the \emph{Mini-Wiki} benchmark introduced in \cite{Khodak19}. This is derived from the Wiki3029 dataset presented in \cite{Arora19}. The dataset is comprised of $4$-class, $m$-shot learning tasks with sample space $\Z = \{z\in \mathbb{R}^d\ | \ \|z\|_2 = 1\}$. Sentences from various Wikipedia articles are passed through the continuous-bag-of-words GloVe embedding \cite{Pennington14} into dimension d = 50 to generate samples. For this learning task, we use a $k$-class version of $\CELs$ and logistic regression. Since this example is convex, we can use GD and bound $\b$ with Theorem \ref{thm:hardt15} in the appendix \cite{Hardt16}. We keep the loss bounded by constraining the network $\|N(z)\|_F \leq r$ and scale the loss as in the previous example. The tightness of the bounds on $c_L$ and $c_S$ affected the upper bound in Inequality \eqref{meta learning bound eval} more than in the previous example, so we bound them as tightly as possible. See Appendix \ref{lip and smooth} for the calculations. 

We apply Algorithm \ref{alg:meta-learning with bound} using the bound which allows for validation data, Inequality \eqref{meta learning bound eval}, to learn on $4$-way \emph{Mini-Wiki} $m=\{1,3,5\}$-shot. The results are presented in Table \ref{classification results}. We compare our results with the FMRL variant which provides a guarantee \cite{Khodak19}, follow-the-last-iterate (FLI)-Batch, and with MR-MAML \cite{Yin20}. FLI-Batch does not require bounded losses explicitly, but requires that the parameters of the network lie within a ball of radius $r$. For the logistic regression used in the example, this is equivalent to $\|N(z)\|_F \leq r$. Thus, we scale the loss and use the same $r$ for each method to provide a fair comparison. We also show the results of training with MAML constrained in the same way for reference. Each method is given the same amount of held-out data for training a prior (see Appendix \ref{ap:expdets_miniwiki} for further details on training the prior).

As in the previous example, PAC-BUS provides a significantly tighter guarantee than the other methods (Table \ref{classification results}). We see similar empirical meta-test loss for MAML \cite{Finn17}, MR-MAML \cite{Yin20}, and PAC-BUS with slightly higher loss for FLI-Batch \cite{Khodak19}. In addition, we computed the meta-test accuracy as the percentage of correctly classified sentences. See Table \ref{miniwiki results score} in Section \ref{ap:expdets_miniwiki} for these results along with other experimental details.

\begin{table*}[t]
\caption{We compare the generalization bounds (for $\delta = 0.01$) provided by each method where applicable and use the sample convergence bound for MR-MAML and PAC-BUS. Since we specifically minimize these methods' upper bounds, we can fairly compare the relative tightness of each bound. We also report the meta-test loss ($\CELs$) for each method for exposition. We report the mean and standard deviation after 5 trials. We highlight that our approach provides the strongest guarantee.}
\label{classification results}
% \vskip 0.15in
\begin{center}
\begin{small}
\begin{tabular}{lcccr}
\toprule
$4$-Way \emph{Mini-Wiki} & $1$-shot $\downarrow$ & $3$-shot $\downarrow$ & $5$-shot $\downarrow$ \\
\midrule
FLI-Batch Bound \cite{Khodak19} & $0.6638\pm0.0011$ & $0.6366\pm0.0006$ & $0.6343\pm0.0014$\\
MR-MAML Bound \cite{Yin20} & $0.7400\pm0.0003$ & $0.7312\pm0.0003$ & $0.7283\pm0.0005$ \\
PAC-BUS Bound (ours) & \textBF{$0.4999\pm0.0003$} & \textBF{$0.5058\pm0.0002$} & \hspace{-0.5em}\textBF{$0.5101\pm0.0002$}\\
\midrule
MAML \cite{Finn17} & \textBF{$0.3916\pm0.0009$} & \textBF{$0.3868\pm0.0005$} & \textBF{$0.3883\pm0.0005$}  \\ 
FLI-Batch \cite{Khodak19} & $0.4091\pm0.0008$ & $0.4078\pm0.0005$ & $0.4097\pm0.0012$ \\
MR-MAML \cite{Yin20} & $0.3922\pm0.0009$ & $0.3869\pm0.0003$ & $0.3884\pm0.0005$ \\
PAC-BUS (ours) & $0.3922\pm0.0009$ & $0.3878\pm0.0003$ & $0.3895\pm0.0005$ \\
\bottomrule
\end{tabular}
\end{small}
\end{center}
\vskip -0.25in
\end{table*}
    
\subsection{Example: memorizable Omniglot}
\label{example_omniglot}
\vspace{-3pt}

We have demonstrated the ability of our approach to provide strong generalization guarantees for meta-learning in the settings above. We now consider a more complex setting where we are unable to obtain strong guarantees. In this example, we employ a learning heuristic based on the PAC-BUS upper bound, \emph{PAC-BUS(H)}; see Appendix \ref{ap:alg_h} for the details and the Algorithm. We relax Assumption \ref{assumption1} and no longer constrain the network as in previous sections. Instead, we maintain and update estimates of the Lipschitz and smoothness constants of the network, using \cite{Sinha20}, and incorporate them into the uniform stability regularizer term, $\b$. We then scale each regularizer term (i.e., $\Rbayes$ and $\b$) by hyper-parameters $\lambda_1$ and $\lambda_2$ respectively. Analogous to the technique described in \cite{Yin20}, we aim to incorporate the form of the theoretically-derived regularizer into the loss, without requiring it to be as restrictive during learning. The result is a regularizer that punishes large deviation from the prior $\Thetap$ and too much adaptation at the base-learning level.

We test our method on \emph{Omniglot} \cite{Lake11} for $20$-way, $m=\{1,5\}$-shot classification in the non-mutually exclusive (NME) case \cite{Yin20}. In \cite{Yin20}, the problem of memorization in meta-learning is explored and demonstrated with non-mutually exclusive learning problems. \emph{NME Omniglot} corresponds to randomization of class labels for a task at test time only. This worsens the performance of any network that memorized class labels; see \cite{Yin20} for more details.$^3$ We compare our method to an analogous heuristic presented in \cite{Yin20}, which also has a $\KL(\Thetaq\|\Thetap)$ term in the loss. Thus, this heuristic (referred to as MR-MAML(W) \cite{Yin20}) regularizes the change in weights of the network. Additionally, we compare to the heuristic described in \cite{Khodak19} (FLI-Online) which performs better in practice than the FLI-Batch method. We do not provide data for training a prior in this case since we do not aim to compute a bound in this example. We use standard MAML as a reference. See Table \ref{omniglot_results} for the results. 

We see that MAML \cite{Finn17} and FLI-Online \cite{Khodak19} do not prevent memorization on \emph{NME Omniglot} \cite{Yin20}. This is especially apparent in the $1$-shot learning case, where their performance suffers significantly due to this memorization. Both MR-MAML(W) \cite{Yin20} and PAC-BUS(H) prevent memorization, with PAC-BUS(H) outperforming MR-MAML(W). Note that PAC-BUS(H) outperforms MR-MAML(W) by a wider margin in the $1$-shot case as compared with the $5$-shot case. We believe this is due to the effectiveness of the uniform stability regularizer at the base level. MR-MAML(W) suffers more in the $1$-shot case because over-adaptation is more likely with fewer within-task examples. \blfootnote{$^3$We use a slightly different task setup as the one in \cite{Yin20}; see Appendix \ref{expdets_omniglot} for the details of our setup.}

\begin{table}
\caption{We present the meta-test accuracy as a percentage on non-mutually-exclusive \textit{Omniglot} \cite{Yin20}. In contrast to the previous examples, here we aim to achieve the best empirical performance for each method. In particular, this task compares each methods' ability to prevent memorization. We report the mean and standard deviation after 5 trials.
\label{omniglot_results}}
% \vskip 0.15in
\begin{center}
\begin{small}
\begin{sc}
\begin{tabular}{lccr}
\toprule
20-Way \emph{Omniglot} & NME 1-shot $\uparrow$ & NME 5-shot $\uparrow$ \\
\midrule
MAML \cite{Finn17}              & $23.4\pm2.2$ & $75.1\pm4.8$ \\
FLI-Online \cite{Khodak19}      & $22.4\pm0.5$ & $39.1\pm0.5$ \\
MR-MAML(W) \cite{Yin20}         & $84.2\pm2.2$ & $94.3\pm0.3$ \\
PAC-BUS(H) (ours)               & \textBF{$87.9\pm0.5$} & \textBF{$95.0\pm0.9$}  \\
\bottomrule
\end{tabular}
\end{sc}
\end{small}
\end{center}
\vspace{-0.12in}
\end{table}
\vspace{-5pt}
\section{Conclusion and discussion}
\label{conclusion}
\vspace{-5pt}

We presented a novel generalization bound for gradient-based meta-learning: PAC-BUS. We use different generalization frameworks for tackling the distinct challenges of generalization at the two levels of meta-learning. In particular, we employ uniform stability bounds and PAC-Bayes bounds at the base- and meta-learning levels respectively. On a toy non-convex problem and the \emph{Mini-Wiki} meta-learning task \cite{Khodak19}, we provide significantly tighter generalization guarantees as compared to state-of-the-art meta-learning bounds while maintaining comparable empirical performance. To our knowledge, this work presents the first numerically-evaluated generalization guarantees associated with a proposed meta-learning bound. On memorizable \emph{Omniglot} \cite{Lake11, Yin20}, we show that a heuristic based on the PAC-BUS bound prevents memorization of class labels in contrast to MAML \cite{Finn17}, and better performance than meta-regularized MAML \cite{Yin20}. We believe our framework is well suited to the few-shot learning problems for which we present empirical results, but our framework is potentially applicable to a broad range of different settings (e.g., reinforcement learning).

We note a few challenges with our method as motivation for future work. Our bound is vacuous on larger scale learning problems such as \emph{Omniglot}. This is partially caused by a larger KL-divergence term in the PAC-Bayes bound when using deep convolutional networks (due to the increased dimensionality of the weight vector). In addition, we do not have a theoretical analysis on the convergence properties of the algorithms presented, so we must experimentally determine the number of samples required for tight bounds. In the results of Section \ref{example_circleclass} and \ref{example_miniwiki}, despite an improved bound over other methods, our method does not necessarily improve empirical test performance. We emphasize that our focus in this work was on deriving stronger generalization guarantees rather than improving empirical performance. However, obtaining approaches that provide both stronger guarantees and empirical performance is an important direction for future work.

Future work can also explore ways in which to incorporate tighter PAC-Bayes bounds or those with less restrictive assumptions. One interesting avenue is to extend PAC-BUS by using a PAC-Bayes bound for unbounded loss functions for the meta-generalization step (e.g. as presented in \cite{Haddouche20}). Another promising direction is to incorporate regularization on the weights of the network directly (e.g., $L_2$ regularization or gradient clipping) to create networks with smaller Lipschitz and smoothness constants. Additionally, it would be interesting to explore learning of the base-learner's algorithm while maintaining uniform stability. For example, one could parameterize a set of uniformly stable algorithms and learn a posterior distribution over the parameters.

\textbf{Broader impact.} The approach we present in this work aims to strengthen performance guarantees for gradient-based meta-learning. We believe that strong generalization guarantees in meta-learning, especially in the few-shot learning case, could lead to broader application of machine learning in real-world applications. One such example is for medical diagnosis, where abundant training data for certain diseases may be difficult to obtain. Another example on which poor performance is not an option is any safety critical robotic system, such as ones which involve human interaction. 

Meta-learning methods typically require a lot of data and training time, and ours is not an exception. In our case, it took multiple weeks of computation time on Amazon Web Services (AWS) instances to train and compute all networks and results we present in this paper. This creates challenges with accessibility and energy usage.

\subsubsection*{Acknowledgments}
The authors are grateful to the anonymous reviewers for their valuable feedback and suggestions, and to Thomas Griffiths for helpful feedback on this work. The authors were supported by the Office of Naval Research [N00014-21-1-2803, N00014-18-1-2873], the NSF CAREER award [2044149], the Google Faculty Research Award, and the Amazon Research Award. 
\newpage
\bibliographystyle{plainnat}
\bibliography{bib.bib}

\newpage
\appendix
\section{Appendix}
\subsection{Proof of Theorem \ref{thm:alg stab gen}}
\label{sec:proofthm1}
\begin{customthm}{\ref{thm:alg stab gen}}[Algorithmic Stability Generalization in Expectation] Fix a task $t \in \T$. The following inequality holds for hypothesis $h_\thp$ learned using $\b$ uniformly stable algorithm $A$ with respect to loss $L$:
\begin{align}
    \EESZ \ \EEtheta \ L\hZ \leq \EESZ \ \EEtheta \ \widehat{L}\hS + \b.
\end{align}
\end{customthm}
\begin{proof}
Let $S = \{z_1, z_2 \dots, z_m \} \sim \Zt^m$ and $S' = \{z_1', z_2' \dots, z_m' \} \sim \Zt^m$ be two independent random samples and let $S^i = \{z_1, \dots, z_{i-1}, z_i', z_{i+1}, \dots, z_m \}$ be identical to $S$ except with the $i^{\text{th}}$ sample replaced with $z_i'$. Fix a distribution $\Thetaq$ over initializations. Consider the following
\begin{align}
        \EESZ \ \EEtheta \ \widehat{L}(h_\thp, S) & = \EESZ \ \EEtheta \bigg[\frac{1}{m} \sum_{i=1}^m L(h_\thp, z_i)\bigg] \\
        & = \EESZ \ \underset{S' \sim \Zt^m}{\EE} \ \EEtheta \bigg[\frac{1}{m} \sum_{i=1}^m L(h_\thpi, z_i)\bigg] \\
        & = \EESZ \ \underset{S' \sim \Zt^m}{\EE} \ \EEtheta \bigg[\frac{1}{m} \sum_{i=1}^m L(h_\thp, z_i')\bigg] + \delta \\
        & = \EESZ \ \EEtheta \ L(h_\thp, \Zt) + \delta
\end{align}
where 
\begin{equation}
    \delta = \EESZ \ \underset{S' \sim \Zt^m}{\EE} \ \EEtheta \bigg[\frac{1}{m} \sum_{i=1}^m L(h_\thpi, z_i') - \sum_{i=1}^m L(h_\thp, z_i')  \bigg].
\end{equation}
We bound $\delta$ with the supremum over datasets $S$ and $S'$ differing by a single sample
\begin{equation}
    \delta \leq \sup_{S, S', z} \ \EEtheta \ [L(h_\thp, z) - L(h_\thp, z)] \leq \b
\end{equation}
by Definition \ref{thm:bousquet}.
\end{proof}

\subsection{Definition of Marginal Distribution $\D_S$}
\label{sec:proofthm3}
In this section we formally define the marginal distribution $\D_S$ which we make use of in the proof of Theorem \ref{thm:meta learning bound}. This is the distribution over datasets one obtains by first sampling a task $t$ from $\T$, and then sampling a dataset $S$ from $\Zt^m$. Consider the following equations (for simplicity, we use summations instead of integrals to compute expectations; $p(t)$ represents the probability of sampling $t$ and $p(S|t)$ is the probability of sampling $S$ given $t$):
\begin{align}
    \EEt \ \EESZ \ \EEtheta \ f(\theta, S) & = \sum_t p(t) \sum_S p(S|t) \EEtheta \ f(\theta, S) \\
    & = \sum_{t,S} p(t) p(S|t) \EEtheta \ f(\theta, S) \\
    & = \sum_{t,S} p(S, t) \EEtheta \ f(\theta, S) \\
    & = \sum_{S} \bigg( \EEtheta \ f(\theta, S) \underbrace{\sum_{t} p(S, t)}_{= p(S)} \bigg).
    % & = \EE_{t\sim P_t} \EE_{S \sim P_{S|t}} \EE_{\theta \sim P_\theta} \ f(\theta, S) \\
    % & = \EE_{(t\times S)\sim (P_t \times P_{S|t})}  \EE_{\theta \sim P_\theta} \ f(\theta, S) \\
    % & = \EE_{(t,S)\sim P_{t,S}}  \EE_{\theta \sim P_\theta} \ f(\theta, S) \\ % second two terms don't depend on t
    % & = \EE_{S}
\end{align}
Here $p(S)$ corresponds to the marginal distribution over datasets $S$. Note that the last line above holds because $\EE_{\theta \sim \Thetaq}f(\theta, S)$ does not depend on $t$.
\begin{definition}[Marginal Distribution Over Datasets $S$]  \label{def:marginaldist}
Let $\D_S := p(S)$ from above.
\end{definition}
Thus we have
\begin{align}
\label{eq:defa1equivalence}
    \EEt \ \EESZ \ \EEtheta \ f(\theta, S) & = \sum_S p(S) \EEtheta \ f(\theta, S) = \EES \ \EEtheta \ f(\theta, S).
\end{align}

\subsection{Specializing the Bound}
\label{ap:bound specialization full}
\subsubsection{Meta-Learning Bound for Gaussian Distributions}
\label{ap:bound for gaussians}

In practice, the distribution $\Thetaq$ over initializations will be a multivariate Gaussian distribution. Thus, in this section, we present a specialization of the bound for Gaussian distributions. % \edit{Let the mean be $\mu$, covariance be $\text{diag}(s)$, such that $\Thetaq = \N(\mu, \text{diag}(s))$. We use the shorthand $\N_\psi$ where $\psi = (\mu, \log(s))$.} \alec{update this and rest of algorithms}
% Let $\Thetaq$ be a multivariate Gaussian distribution over initializations $\theta$ with mean $\mu$ and covariance $\text{diag}(s)$; thus $\Thetaq = \N(\mu,\text{diag}(s))$ and $\Thetap = \N(\mu_0,\text{diag}(s_0))$. Further, let $\psi := (\mu, \log(s))$, and use the shorthand $\N_{\psi_0}$ for the prior and $\N_\psi$ for the posterior distribution over initializations.
Let $\Thetaq$ have mean $\mu$ and covariance $\Sigma$; thus $\Thetaq = \N(\mu, \Sigma)$ and analogously $\Thetap = \N(\mu_0,\Sigma_0)$. 
We can then apply the analytical form for the KL-divergence between two multivariate Gaussian distributions to the bound presented in Theorem \ref{thm:meta learning bound}. The result is the following bound holding under the same assumptions as Theorem \ref{thm:meta learning bound}:
\begin{align}
    \label{meta learning bound gaussians}
    \LTT \leq & \ \frac{1}{l} \sum_{i=1}^l \ \EEtheta \ L(h_\thpi, S_i) + \b \nonumber \\
    & + \sqrt{\frac{(\mu - \mu_0)\Sigma_0^{-1}(\mu - \mu_0) + \ln\frac{|\Sigma_0|}{|\Sigma|} + \text{tr}(\Sigma_0^{-1}\Sigma) - n_\text{dim} + 2\ln \frac{2\sqrt{l}}{\delta}}{4l}},
\end{align}
where $n_\text{dim}$ is the number of dimensions of the Gaussian distribution. We implement the above bound in code instead of the non-specialized form of the KL divergence to speed up computations and simplify gradient computations.

\subsubsection{Few-Shot Learning Bound with Validation Data}
\label{ap:specialized bound}
In this section, we will assume that, in addition to the training data $S \sim \Zt^m$, we have access to validation data $S_\va \sim \Zt^n$ at meta-training time. We will show that a meta-learning generalization bound can still be obtained in this case. Notably, this will not require validation data at meta-testing time.

We begin by bounding the expected loss on evaluation data $S_\ev = \{S, S_\va \}$ after training on $S$. Note that for other meta-learning techniques, the training data $S$ is often excluded from the data used to update the meta-learner. Including it here helps relate the loss on $\Zt$ to the loss on $S_\ev$ after adaptation with $S$ (see derivation below), and is necessary to achieve a guarantee on performance for the few-shot learning case. From Inequality \eqref{gen pac bayes pre}, we set the arbitrary distribution $\D_s$ to the marginal distribution $\D_{S_\ev}$ over datasets of size $m+n$ and $f(\theta, s) := \widehat{L}\hSev$. Note that with this marginal distribution, we have an equivalence of sampling given by 
\begin{equation}
    \EESev[\cdot] = \EEt \ \EESevZ[\cdot].
\end{equation}
The following inequality holds with high probability over a sampling of $\textbf{S}_\ev = \{S_{\ev,1},S_{\ev,2},\dots, S_{\ev,l}\} \ \sim \D_{S_\ev}^l$:
\begin{align}
    \EESev \ \EEtheta \widehat{L}\hSev & = \nonumber \\ 
    \EEt \ \EESevZ \ \EEtheta \ \widehat{L}\hSev &\leq  \frac{1}{l} \sum_{i=1}^l \ \EEtheta \ \widehat{L}\hiSev + \Rbayes \label{eq:equivalence_sev_sevz}.
\end{align}
In the next steps, we aim to isolate for a $\widehat{L}\hS$ term so that we may still combine with Inequality \eqref{maurerineq} as we did in Section \ref{sec:meta-learning bound}. We decompose the LHS of Inequality \eqref{eq:equivalence_sev_sevz},
\begin{equation}
    \frac{1}{m+n} \ \EEt \Big{[}m \EESZ \ \EEtheta \ \widehat{L}\hS + n \EESZ \ \EESvaZ \ \EEtheta \ \widehat{L}\hSva \Big{]}. \label{eq:decomp of lhs in eq15}
\end{equation}
Since the validation data $S_\va$ is sampled independently from $S$, the expected training loss on the validation data is the true expected loss over sample space $\Zt$,
\begin{equation}
    \EESZ \ \EESvaZ \ \EEtheta \ \widehat{L}\hSva = \EESZ \ \EEtheta \ L\hZ. \label{eq:exp value Sva Z equivalence}
\end{equation}
We plug Equality \eqref{eq:exp value Sva Z equivalence} into Equation \eqref{eq:decomp of lhs in eq15}, and then the decomposition in Equation \eqref{eq:decomp of lhs in eq15} into Inequality \eqref{eq:equivalence_sev_sevz}. We can then isolate for the $\widehat{L}\hS$ term,
\begin{align}
    \EEt \EESZ \EEtheta \widehat{L}\hS \leq \ & \frac{m+n}{m} \Bigg{[}\frac{1}{l} \sum_{i=1}^l \EEtheta \widehat{L}\hiSev + \Rbayes  \Bigg{]} \nonumber \\ &  - \frac{n}{m} \EEt \ \EESZ \ \EEtheta \ L\hZ,
\end{align}
and plug into the LHS of Equation \eqref{maurerineq}. By simplifying, we find that
\begin{equation}
    \label{ap:meta learning bound eval}
    \EEt \EESZ \EEtheta L\hZ \leq \frac{1}{l} \sum_{i=1}^l \EEtheta \widehat{L}\hiSev + \Rbayes + \frac{m\b}{m+n}.
\end{equation}
This resulting bound is very similar to the one in Inequality \eqref{meta learning bound}. We compute the loss term in the upper bound with evaluation data and as a result, the size of the uniform stability regularization term is reduced. 

\subsection{Bounds on the Uniform Stability Constant}
\label{ap:beta_bounds}
In this section, we present bounds from \cite{Hardt16} on the uniform stability constant $\b$ which are applicable to our settings. We first formalize the definitions of Lipschitz continuous (``Lipschitz" with constant $c_L$) and Lipschitz smoothness (``smooth" with constant $c_S$). 
\begin{definition}[$c_L$-Lipschitz] Function f is $c_L$-Lipschitz if $\forall \ \theta,\theta' \in \thetaspace, \forall \ z\in \Z$ the following holds: \label{def:lipschitz}
\begin{equation}
    |f(\theta,z) - f(\theta',z)| \leq c_L\|\theta - \theta\|.
\end{equation}
\end{definition}
\begin{definition}[$c_S$-smooth] Function f is $c_S$-smooth if $\forall \ \theta, \theta' \in \thetaspace, \forall \ z\in \Z$ the following holds: \label{def:smooth}
\begin{equation}
\|\nabla f(\theta,z) - \nabla f(\theta',z)\| \leq c_S\|\theta - \theta'\|.
\end{equation}
\end{definition}

Using a convex loss and stochastic gradient descent (SGD) allows us to directly bound the uniform stability constant $\b$ \cite{Hardt16}:
\begin{theorem}[Convex Loss SGD is Uniformly Stable \cite{Hardt16}] Assume that convex loss function $L$ is $c_S$-smooth and $c_L$-Lipschitz $\forall \ z \in \Z$. Suppose we run SGD on $S$ with step size $\alpha \leq \frac{2}{c_S}$ for $T$ steps. Then SGD satisfies $\b$ uniform stability with \label{thm:hardt15}
\begin{equation} 
\label{beta bound sgm}
    \b \leq \frac{2c_L^2}{m}T\alpha.
\end{equation}
\end{theorem}

Note that the bounds on $\b$ presented in \cite{Hardt16} guarantee $\b$ uniform stability in expectation for a randomized algorithm A. However, for deterministic algorithms, this reduces to $\b$ uniform stability. Using the uniform stability in expectation definition introduces another expectation (over a draw of algorithm $A$) into the upper bound of the meta-learning generalization guarantee in Inequality \eqref{meta learning bound}. So as to not increase the computation required to estimate the upper bound, we let $A$ be deterministic. This is achieved either by fixing the order of the samples on which we perform gradient updates for SGD, or by using gradient descent (GD). Additionally, in the convex case, $T$ steps of GD satisfies the same bound on $\b$ as $T$ steps of SGD; see Appendix \ref{ap:uniform stability gd} for the proof. For non-convex losses, a bound on $\b$ is still achieved when algorithm $A$ is SGD \cite{Hardt16}:
\begin{theorem}[Non-Convex Loss SGD is Uniformly Stable \cite{Hardt16}] 
\label{thm:nonconvex_betabound} Let non-convex loss $L$ be $c_S$-smooth and $c_L$-Lipschitz $\forall \ z \in \Zt$ and satisfy Assumption \ref{assumption1}. Suppose we run $T$ steps of SGD with monotonically non-increasing step size $\alpha_t \leq \frac{c}{t}$. Then SGD satisfies $\b$ uniform stability with
\begin{equation}
    \b \leq \frac{1 + \frac{1}{ c_S c}}{n-1}(2 c_L^2 c)^{\frac{1}{ c_S c+1}}T^{\frac{c_S c}{c_S c + 1}}
\end{equation}
\end{theorem}
\noindent Note that this bound does not hold when GD is used.

\subsection{Algorithms}
\label{ap:algorithms}
Before running the algorithms presented in this paper, we must deal with a few technical challenges that arise from our method's assumptions and terms which need to be computed. In this paragraph, we discuss the approach we take to deal with these challenges. For arbitrary networks, the softmax-activated cross entropy loss ($\CELs$) is not bounded and would not satisfy Assumption \ref{assumption1}. We thus constrain the network parameters to lie within a ball and scale the loss function such that all samples $z \in \Z$ achieve a loss within $[0,1]$; see Appendix \ref{scale shift} for details. However, the PAC-BUS framework works with distributions $\Thetaq$ over initializations. One option is to let $\Thetaq$ be a projected multivariate Gaussian distribution. This prevents the network's output from becoming arbitrarily large. However, the upper bound in Inequality \eqref{meta learning bound eval} requires the KL-divergence between the prior and posterior distribution over initializations. This is difficult to calculate for projected multivariate Gaussian distributions and would require much more computation during gradient steps. Since the KL-divergence between projected Gaussians is less than that between Gaussians (due to the data processing inequality \cite{Cover12}), we can loosen the upper bound in \eqref{meta learning bound} and \eqref{meta learning bound eval} by computing the upper bound using the non-projected distributions (but using the projected Gaussians for the algorithm). After sampling a base learner's initialization, we re-scale the network such that its parameters lie within a ball of radius $r$. We also re-scale the base learner's parameters after each gradient step to guarantee that the loss stays within $[0,1]$. Projection after gradient steps is not standard SGD, but we show that it maintains the same bound on $\b$; see Section \ref{ap:stab_constant_considerations} for details of the proof. Thus, we let algorithm $A$ be SGD with projections after each update and use Theorem \ref{thm:nonconvex_betabound} to bound $\b$ for non-convex losses \cite{Hardt16}. Additionally, we can upper bound the Lipschitz $c_L$ and smoothness $c_S$ constants for the network using the methods presented in \cite{Sinha20}. After working through these technicalities, we can compute all terms in the upper bound.

\subsubsection{PAC-BUS using Mini-Batches of Tasks} \label{ap:minibatches}
We present the PAC-BUS algorithm modified for mini-batches of tasks to improve training times. For batches of size $k$, the algorithm is presented in \ref{alg:minibatches}.

\begin{algorithm}[tb]
    \caption{PAC-BUS using Mini-Batches of Tasks}
    \label{alg:minibatches}
\begin{algorithmic}
    \State \textbf{Input}: Fixed prior distribution $\N_{\psi_0}$ over initializations 
    \State \textbf{Input}: $\b$ uniformly stable Algorithm $A$ 
    \State \textbf{Input}: Meta-training dataset $\textbf{S} = \{S_1, S_2,\dots, S_l\}$, learning rate $\gamma$ 
    \State \textbf{Initialize}: $\psi \leftarrow \psi_0$ 
    \State \textbf{Output}: Optimized $\psi^*$
    \State $B(\psi, \theta'_1,\theta'_2 \dots, \theta'_k) := \frac{1}{l} \sum_{i=1}^k \widehat{L}(h_{\theta'_i},S_i) + \RbayesN +\b$
    \While{not converged}
    \State Sample $\theta \sim \N_\psi$
    \For{$i=1$ {\bfseries to} $k$}
    \State $j \sim \text{Uniform}\{1, 2, \dots, l\}$
    \State $\theta'_i \leftarrow A(\theta, S_j)$ 
    \EndFor
    \State $\psi \leftarrow \psi - \gamma \nabla_{\psi} B(\psi, \theta'_1,\theta'_2 \dots, \theta'_k)$
    \EndWhile
\end{algorithmic}
\end{algorithm}

\subsubsection{PAC-BUS(H)}
\label{ap:alg_h}
In addition to providing algorithms which minimize the upper bound in Inequalities \eqref{meta learning bound} and \eqref{meta learning bound eval}, we are also interested in a regularization scheme which re-weights the regularizer terms in these bounds. For larger scale and complex settings, it is challenging to provide a non-vacuous guarantee on performance, but weighting regularizer terms has been shown to be an effective training technique \cite{Yin20}. We calculate $\b$ with a one-gradient-step version of Theorem \ref{thm:nonconvex_betabound}. This Theorem requires the algorithm $A$ to be SGD, but we let $A$ be a single step of GD to improve training times. We also relax Assumption \ref{assumption1} Since the $\b$ depends on both $c_L$ and $c_S$, we update estimates of them after each iteration by sampling multiple $\theta \sim \Thetaq$, bound the $c_L$ and $c_S$ for those sets of parameters using Section 4 of \cite{Sinha20}, and then choose the maximum to compute $\b$. This is in contrast to limiting the network parameters directly by bounding the output of the loss. Instead, the $\b$ term in the upper bound and the scale factor will determine how much to restrict the network parameters. The resulting method is presented in Algorithm \ref{alg:meta-learning with bound (H)}. In order to provide strong performance in practice, we tune $\lambda_1$ and $\lambda_2$. % \edit{This algorithm can also be modified to learn the learning rate $\alpha$ by augmenting $\psi$ to include a parameterization over $\alpha$.}

\begin{algorithm}[tb]
    \caption{PAC-BUS(H):  Meta-learning heuristic based on PAC-BUS upper bound}
    \label{alg:meta-learning with bound (H)}
\begin{algorithmic}
    \State \textbf{Input}: Fixed prior distribution $\N_{\psi_0}$ over initializations
    \State \textbf{Input}: Meta-training dataset $\textbf{S}$, learning rates $\alpha$ and $\gamma$
    \State \textbf{Input}: Scale factors $\lambda_1,\lambda_2$ for regularization terms
    \State \textbf{Initialize}: $\psi \leftarrow \psi_0$
    \State \textbf{Output}: Optimized $\psi^*$
    \State $B(\psi, c_L, c_S, \theta'_1,\theta'_2,\dots,\theta'_l) := \frac{1}{l} \sum_{i=1}^l \widehat{L}(h_{\theta'_i},S_i) + \lambda_1 \RbayesN + \lambda_2 \b(c_L,c_S)$
    \State Estimate $c_L$ and $c_S$ using $\N_{\psi_0}$
    \While{not converged}
    \State Sample $\theta \sim \N_\psi$
    \For{$i=1$ {\bfseries to} $l$}
    \State $\theta'_i \leftarrow \theta - \alpha\nabla_\theta \widehat{L}(h_\theta,S_i)$
    \EndFor
    \State $\psi \leftarrow \psi - \gamma \nabla_{\psi} B(\psi, c_L, c_S, \theta'_1, \theta'_2, \dots, \theta'_l)$
    \State Estimate $c_L$ and $c_S$ using $\N_\psi$
    \EndWhile
\end{algorithmic}
\end{algorithm}

\subsection{Sample Convergence Bound}
\label{ap:sample convergence bound}

After training is complete, we aim to compute the upper bound. However, this requires evaluating an expectation $\theta \sim \Thetaq$, which may be intractable. Providing a valid PAC guarantee without needing to evaluate the expectation taken over $\theta \sim \Thetaq$ requires the use of the sample convergence bound \cite{Langford02}. We have the following guarantee with probability $1 - \delta'$ over a random draw of $\{\theta_1,\theta_2 \dots, \theta_N\} \sim \Thetaq^N$ for any dataset $S$ \cite{Langford02},
\begin{align}
    \label{sample_convergence_bound}
    \KL\Bigg{(}\sum_{j=1}^N L(h_{A(\theta_j,S)},\cdot) \Bigg{\|}\EEtheta L(h_\thp, \cdot)\Bigg{)} \leq \frac{\log(\frac{2}{\delta'})}{N}.
\end{align}
We can invert this KL-style bound (i.e. a bound of the form $\KL(p\|q^*)\leq c$) by solving the optimization problem, $q^* \leq \KL^{-1}(q \| c) := \sup\{q\in [0,1] : \KL(p\|q) \leq c\}$, as described in \cite{Dziugiate17}. After the inversion is performed on Inequality \eqref{sample_convergence_bound}, we use a union bound to combine the result with Inequality \eqref{meta learning bound} and retain a guarantee with probability $1 - \delta - \delta'$ as in \cite{Dziugiate17},
\begin{align}
    \LTT \leq &\ \frac{1}{l} \sum_{i=1}^l \KL^{-1}\bigg{(} \sum_{j=1}^N L(h_{A(\theta_j,S_i)},S_i) \bigg{\|}\frac{\log(\frac{2}{\delta'})}{N}\bigg{)} + \Rbayes + \b. 
    \label{meta learning bound full}
\end{align}
An analogous bound is achieved when combined with Inequality \eqref{meta learning bound eval}. Thus, after training, we evaluate Inequality \eqref{meta learning bound full} to provide the guarantee. Note that use of the sample convergence bound is a loosening step. However, in our experiments, the upper bound in Inequality \eqref{meta learning bound full} is less than 5\% looser than unbiased estimates of Inequality \eqref{meta learning bound}. This can be reduced further at the expense of computation time (if we utilize a larger number of samples in the concentration inequality).

\subsection{Constraining Parameters and Scaling Losses}
\label{scale shift}

In order to maintain a guarantee, the PAC-Bayes upper bound in Theorem \ref{thm:gen pac bayes} requires a loss function bounded between $0$ and $1$. However, the losses we use are not bounded in general. Let $N_\theta$ be an arbitrary network parameterized by $\theta$ and $N_\theta(z)$ be the output of the network given sample $z \in \Z$. Consider arbitrary loss $f$, which maps the network's output to a real number. If $\|N_\theta(z)\| \leq r, \forall \ \theta \in \thetaspace, \forall \ z \in \Z$, then we can perform a linear scaling of $f$ to map it onto the interval $[0,1]$. We define the minimum and maximum value achievable by loss function $f$ as follows
\begin{align}
    M_f & := \max_{z\in \Z, \ \theta \in \thetaspace ,\ \|N_\theta(z)\| \leq r}f(\theta, z) \\
    m_f & := \min_{z\in \Z, \ \theta \in \thetaspace ,\ \|N_\theta(z)\| \leq r}f(\theta, z).
\end{align}
Now we can define a scaled function
\begin{equation}
    f_S(\theta,z):= \frac{f(\theta,z) - m_f}{M_f - m_f}
\end{equation}
such that $f_S(\theta,z) \in [0,1]$. Note that the Lipschitz and smoothness constants of $f_S$ are also scaled by $\frac{1}{M_f - m_f}$. When we choose loss $\CELs$, the $k$-class cross entropy loss with softmax activation, we have 
\begin{align}
    M_{\CELs}:= \log\Big{(}\frac{e^{-r}+(k-1)}{e^{-r}}\Big{)} , \ \ \ 
     m_{\CELs}:= \log\Big{(}\frac{e^{r}+(k-1)}{e^{r}}\Big{)}.
\end{align}
However, we must restrict the parameters in such a way that satisfies $\|N_\theta(z)\| \leq r$. For arbitrary networks structures, this is not straightforward, so we only analyze the case we use in this paper. Consider an $L$-layer network with ELU activation. Let parameters $\theta$ contain weights $\bW_1, \dots, \bW_L$, and biases $b_1,\dots, b_L$, and assume bounded input $\|z\| \leq r_z, \forall \ \Z$.
\begin{align}
    \|N_\theta(z)\| & = \| \text{ELU}\big{(} \bW_L \text{ELU}(\bW_{L-1}(\cdots) + b_{L-1}) + b_{L}\big{)}\| \\ 
    & \leq \| \bW_L\|_F (\|\bW_{L-1}\|_F (\cdots) + \|b_{L-1}\|) + \|b_L\| \leq r \label{eq:Llayernetworknorm}
\end{align}
We can satisfy $\|N_\bW(z)\| \leq r$ by restricting
\begin{equation}
    \|\theta\|^2 = \sum_{i=1}^L\|\bW_i\|_F^2 + \sum_{i=1}^L\|b_i\|^2 \leq \bigg{(}\frac{r}{\max(1,r_z)}\bigg{)}^2. \label{eq:a_mean}
\end{equation}
Equation \eqref{eq:a_mean} implies Equation \eqref{eq:Llayernetworknorm} by applying the inequality of arithmetic and geometric means. Thus, we ensure $\|\theta\| \leq r / \max(1, r_z)$ by projecting the network parameters onto the ball of radius $\min(r,\frac{r}{r_z})$ after each gradient update.

\subsection{Uniform Stability Considerations}
    \subsubsection{Uniform Stability for Gradient Descent}
\label{ap:uniform stability gd}

In this section, we will prove that $T$ steps of GD has the same uniform stability constant as $T$ steps of SGD in the convex case. This will allow us to use GD when attempting to minimize a convex loss, Section \ref{example_miniwiki}. Let the gradient update rule $G$ be given by $G(\theta,z) = \theta - \alpha \nabla_{\theta}f(\theta, z)$ for \emph{convex} loss function $f$, initialization $\theta \in \thetaspace$, sample $z \in \Z$, and positive learning rate $\alpha$. We define two key properties for gradient updates: expansiveness and boundedness \cite{Hardt16}. 
\begin{definition}[$c_E$-expansive, Definition 2.3 in \cite{Hardt16}] Update rule $G$ is $c_E$-expansive if $\forall \ \theta, \theta' \in \thetaspace, \forall \ z \in \Z$ the following holds: \label{def:expansive} 
\begin{equation}
    \label{eq:expansive}
    \| G(\theta,z) - G(\theta',z)\| \leq c_E \|\theta - \theta' \|.
\end{equation}
\end{definition}
\begin{definition}[$c_B$-bounded, Definition 2.4 in \cite{Hardt16}] Update rule $G$ is $c_B$-bounded if  $\forall \ \theta \in \thetaspace, \forall \ z \in \Z$ the following holds: \label{def:bounded}
\begin{equation}
    \label{eq:bounded}
    \| \theta - G(\theta,z) \| \leq c_B.
\end{equation}
\end{definition}

Now, consider dataset $S \in \Z^m$ and define $\bar{f}(\theta, S) := \frac{1}{m}\sum_{i=1}^m f(\theta, z_i)$. We also define $\bar{G}(\theta,S) :=\theta - \alpha \nabla_\theta \bar{f}(\theta, S) = \sum_{i=1}^m G(\theta, z_i)$. Assume that $G(\theta, z)$, is $c_E$-expansive and $c_B$-bounded $\forall \ z \in \Z$. We then bound the expansiveness of $\bar{G}(\theta,S)$,
\begin{equation}
    \|\bar{G}(\theta,S) - \bar{G}(\theta',S)\|  \leq  \frac{1}{m}\sum_{i=1}^m\|G(\theta,z_i) - G(\theta',z_i)\| \leq \frac{1}{m}\sum_{i=1}^m c_E\|\theta - \theta' \| = c_E\|\theta - \theta' \|.
\end{equation}
To compute the boundedness, consider
\begin{equation}
    \|\theta - \bar{G}(\theta,S)\| \leq \frac{1}{m}\sum_{i=1}^m \| \theta - G(\theta,z_i) \| \leq \frac{1}{m}\sum_{i=1}^m c_B = c_B.
\end{equation}
For a single gradient step on sample $z$, we see the same bounds on $c_E$ and $c_B$ when performing a single GD step on dataset $S$. Thus, if Lemmas $2.5$, $3.3$, and $3.7$ in \cite{Hardt16} are true for gradient updates $G$, they are also true for gradient updates $\bar{G}$. We can then run through the proof of Theorem 3.8 in \cite{Hardt16} to show that it holds for $T$ steps of GD if it holds for $T$ steps of SGD. 

Let $S \sim \D_S$ be a dataset of size $m$ and $S'$ be an identical dataset with one element changed. We run $T$ steps of GD updates, $\bar{G}$, on each of $S$ and $S'$. This results in parameters $\theta_1, \dots, \theta_T$ and $\theta_1', \dots, \theta_T'$ respectively. Fix learning rate $\alpha \leq \frac{2}{c_S}$ and consider 
\begin{align}
    \underset{S,S'}{\EE} \| \theta_{t+1} - \theta_{t+1}'\| & = \underset{S,S'}{\EE} \| \bar{G}(\theta_t,S)  - \bar{G}(\theta_t', S')\| \\ 
    & \leq \frac{1}{m} \sum_{j=1, i\neq j}^m  \underset{S,S'}{\EE}  \| G(\theta_t,z_j)  - G(\theta_t', z_j)\| + \frac{1}{m} \ \underset{S,S'}{\EE}  \| G(\theta_t,z_i)  - G(\theta_t', z_i))\| \\
    & \leq \frac{m-1}{m} \ \underset{S,S'}{\EE} \| \theta_{t} - \theta_{t}'\| + \frac{1}{m} \ \underset{S,S'}{\EE} \| \theta_{t} - \theta_{t}'\| + \frac{2\alpha c_L}{m} = \underset{S,S'}{\EE} \| \theta_{t} - \theta_{t}'\|+ \frac{2\alpha c_L}{m}
\end{align}
The steps above follow from Lemmas 2.5, 3.3, and 3.7 in \cite{Hardt16} and the linearity of expectation. The rest of the proof follows naturally and results in a uniformly stable constant $\b \leq \frac{2c_L^2}{m}T\alpha$ for $T$ steps of GD. Thus, we have the following result.
\begin{corollary} Assume that loss convex function $f$ is $c_S$-smooth and $c_L$ Lipschitz $\forall \ z \in \Z$. Suppose $T$ steps of SGD on $S$ satisfies $\b$ uniform stability. This implies that $T$ steps of GD on $S$ satisfies $\b$ uniform stability. 
\end{corollary}
    \subsubsection{Uniform Stability Under Projections}
\label{ap:stab_constant_considerations}

Projecting parameters onto a ball after gradient updates does not constitute standard SGD nor GD, so we analyze the stability constant after $T$ steps of $G_P(\theta,z) = \text{Proj}[\theta - \alpha \nabla_\theta f(\theta,z)]$. Assume $\|z\| \leq r_z, \forall \ z \in \Z$. The function $\text{Proj}$ scales parameters to satisfy $\|\theta\| \leq \max(r,\frac{r}{r_z})$ if it is not already satisfied. See Appendix \ref{scale shift} for an explanation of this restriction. 

As in Appendix \ref{ap:uniform stability gd}, we compute bounds on the expansiveness and boundedness of $G_P$. Suppose $\theta$ is a vector containing all weights of an $L$-layer network. Network hyper-parameters such as learning rate and activation parameters do not need to be projected, so they will not be included. Assume that $\theta, \theta'$ already satisfy $\|N_\theta(z)\| \leq r, \forall \ z \in \Z$. Consider
\begin{equation}
    \| G_P(\theta,z) - G_P(\theta',z)\| = \| \text{Proj}(G(\theta,z)) - \text{Proj}(G(\theta',z))\| \leq  \| G(\theta,z) - G(\theta',z)\| \leq c_E\|\theta- \theta'\|.
\end{equation}
Note that any required scaling is equivalent to orthogonal projection of the parameters onto a euclidean norm ball of radius $r$ in $R^{d}$, where $d$ is the number of parameters in the network. Thus, the first inequality follows from the fact that orthogonal projections onto closed convex sets satisfy the contractive property \cite{Schneider13}. Next, consider
\begin{align}
    \|\theta - G_P(\theta,z)\| = \|\text{Proj}(\theta) - \text{Proj}(G_P(\theta,z))\| \leq \|\theta - G(\theta,z)\| \leq c_B. 
\end{align}
The equality follows from the assumption that $\theta$ already satisfies the norm constraint. As above, the first inequality follows from the fact that the $\text{Proj}$ function satisfies the contractive property  \cite{Schneider13}. %\cite[Theorem 1.2.1]{Schneider13}.

With these bounds, gradient update $G_P$ satisfies Lemmas 2.5, 3.3, and 3.7 from \cite{Hardt16} if $G$ does. Note that an analogous procedure can be used to show that scaling after a GD update, $\bar{G}_P$, also satisfies these Lemmas. When function $f$ or $\bar{f}$ is convex, the proof of Theorem 3.8 in \cite{Hardt16} applies, and shows that using gradient updates $G_P$ or $\bar{G}_P$ achieve the same bound on the uniform stability constant $\b$. Thus, when $f$ is convex, we may use $G_P$ or $\bar{G}_P$ to compute updates and maintain the guarantee presented in Theorem \ref{thm:bousquet}. Suppose now that $f$ is not convex. Using Lemmas 2.5, 3.3, 3.7, and 3.11 from \cite{Hardt16}, the proof of Theorem 3.12 in \cite{Hardt16} follows naturally to achieve a bound on SGD using projected gradient updates $G_P$ when $f$ is not convex. 

\subsection{Lipschitz and Smoothness Constant Calculation}
\label{lip and smooth}
Recall Definitions \ref{def:lipschitz} and \ref{def:smooth} for a function which is $c_L$-Lipschitz and $c_S$-smooth from Appendix \ref{ap:beta_bounds}. We define the softmax activation function.
\begin{definition}[Softmax Function] $s:\mathbb{R}^k \rightarrow \mathbb{P}^k$
\begin{equation}
    s(u)_i = \frac{e^{u_i}}{\sum_{j=1}^k e^{u_j}} \ , \ \forall \ i .
\end{equation}
\end{definition}
Where every element in $\mathbb{P}^k$ is a probability distribution in $k$ dimensions (i.e. if $v \in \mathbb{P}^k$, then $\sum_{i=1}^k v_i=1$ and $v_i \geq 0 \ \forall \ i$). Since the stability constant $\b$ depends directly on the Lipschitz constant of the loss function, and $\b$ appears in the regularizer of the final bound, we will be as tight as possible when bounding the Lipschitz constant to keep the generalization as tight as possible. Section 6.2 of \cite{Watt16} describes an approach for bounding the Lipschitz constant for the 2-class, sigmoid activated, cross entropy loss. We are interested in the k-class case with softmax activation, and also aim to bound the smoothness constant. We begin with a similar analysis to the one described in \cite{Watt16}.

Given unit-length column vector $z \in \mathbb{R}^d$ and row vector $y \in \mathbb{P}^k$, with weight matrix $\bW \in \mathbb{R}^{d\times k}$ (representing a single-layer network), the loss function is given by:
\begin{equation}
	\CELs(\bW) =  -\sum_{i=1}^k y_i \log(s(z^T\bW)_i).
\end{equation}
Note that while $y$ is any probability distribution, in practice, $y$ will be an indicator vector, describing the correct label with a $1$ in the index of the correct class and $0$ elsewhere. However, the analysis that follows does not depend on this assumption. 

We will take the Hessian of this loss to determine convexity and the Lipschitz constant. However, since the weights are given by a matrix, the Hessian would be a 4-tensor. To simplify the analysis, we will define 
\begin{equation}
	\bw = \begin{bmatrix}
		\bW_{:,1} \\ \bW_{:,2} \\ \vdots \\ \bW_{:,k}
	\end{bmatrix} .
\end{equation}
Where $\bW_{:,i}$ is the $i^{th}$ column of $\bW$ such that $\bw \in \mathbb{R}^{dk}$. We also let
\begin{equation}
	\bzi^T = \begin{bmatrix}
		\bar{0} & \dots & \bar{0} & z^T & \bar{0} & \dots & \bar{0} 
	\end{bmatrix} 
\end{equation}
such that $z$ is placed in the $i^{th}$ group of $d$ elements and $\bar{0}$ is a row vector of $d$ zeros. Vector $\bzi \in \mathbb{R}^{dk}$ since there are $k$ groups. With these definitions, we write the softmax activated network defined by $\bW$ with input $z$:
\begin{equation}
    s(z^T\bW)_i = \frac{e^{\bzi^T\bw}}{\sum_{j=1}^k e^{\bzj^T\bw}}.
\end{equation}
We can simplify this by plugging in for the definition of $s$:
\begin{align}
    \CELs(\bw) := \CELs(\bW) &= - \sum_{i=1}^k y_i\Bigg{[}\bzi^T\bw - \log\bigg{(}\sum_{j=1}^k e^{\bzj^T\bw}\bigg{)}\Bigg{]} \\ 
    &= -\sum_{i=1}^ky_i\bzi^T\bw + \log\bigg{(}\sum_{i=1}^k e^{\bzi^T\bw}\bigg{)}.
\end{align}
These are equivalent because $\sum_{i=1}^k y_i = 1$. For readability, we let $p_i := s(z^T\bW)_i$. With these preliminaries the Hessian will be a 2-tensor and the $\nabla_\bw^3$ term will be a 3-tensor. We compute the gradient and Hessian and $\nabla_\bw^3$ term:
\begin{align}
    	\nabla_\bw \CELs(\bw) &= -\sum_{i=1}^k y_i \bzi + \sum_{i=1}^k\bzi p_i \\
    	\nabla_\bw^2 \CELs(\bw) &= \sum_{i=1}^k \bzi \bzi^T p_i - \bigg{(}\sum_{i=1}^k \bzi p_i\bigg{)}\bigg{(}\sum_{j=1}^k \bzj^T p_j\bigg{)}.
\end{align}
We write $\nabla_\bw^3 \CELs(\bw)$ termwise to simplify notation:
\begin{align}
    \nabla_\bw^3 \CELs(\bw) = \begin{cases} 
      (p_i - 3p_i^2 + 2p_i^3)z\otimes z^T \otimes z^\perp & i=j=l \\
      (-p_ip_l + 2p_i^2p_l) z\otimes z^T \otimes z^\perp & i=j \neq l \\
      (-p_jp_i + 2p_j^2p_i)z\otimes z^T \otimes z^\perp & j=l \neq i \\
      (-p_lp_j + 2p_l^2p_j)z\otimes z^T \otimes z^\perp & l=i \neq j \\
      (2p_i p_j p_l)z\otimes z^T \otimes z^\perp & i \neq j \neq l
   \end{cases}
\end{align}
Where $\otimes$ is the tensor product and $z\otimes z^T \otimes z^\perp \in \mathbb{R}^{d \times d \times d}$ is a 3-tensor with the abuse of notation: $z \in \mathbb{R}^{d \times 1 \times 1}$, $z^T \in \mathbb{R}^{1 \times d \times 1}$, and $z^\perp \in \mathbb{R}^{1 \times 1 \times d}$. Thus $\nabla_\bw^3 \CELs(\bw) \in \mathbb{R}^{dk \times dk \times dk}$.

For twice-differentiable functions, the Lipschitz constant is given by the greatest eigenvalue of the Hessian. Correspondingly, the smoothness constant is given by the greatest eigenvalue of the $\nabla_\bw^3$ term for thrice-differentiable functions. Thus, we aim to bound the largest value that the Rayleigh quotient can take for any unit-length vector $x$. For the Hessian:
\begin{align}
    x^T \nabla_\bw^2 \CELs(\bw) x &\leq | x^T \nabla_\bw^2 \CELs(\bw) x| = \|  x^T \nabla_\bw^2 \CELs(\bw) x\|_F  \\
    &\leq \|x\|^2 \|\nabla_\bw^2 \CELs(\bw)\|_F = \|\nabla_\bw^2 \CELs(\bw)\|_F \\
    &= \sqrt{\sum_{i=1}^k \|zz^T\|_F(p_i - p_i^2)^2 + \sum_{i=1}^k\sum_{j=1,j\neq i}^k \|zz^T\|_F(p_i p_j)^2} \\
    &= \sqrt{\sum_{i=1}^k (p_i - p_i^2)^2 + \sum_{i=1}^k\sum_{j=1,j\neq i}^k (p_i p_j)^2}.
\end{align}
The Frobenius norm is maximized when $p_i = \frac{1}{k}$ for $k > 1$:
\begin{align}
    \|\nabla_\bw^2 \CELs(\bw)\|_F &\leq \sqrt{k\Big{(}\frac{1}{k} - \frac{1}{k^2}\Big{)}^2 + k(k-1)\Big{(}\frac{1}{k^2}\Big{)}^2} \\
    &=\frac{\sqrt{k-1}}{k}.
\end{align}
Thus, for $\CELs(\bw)$, the Lipschitz constant, $c_L \leq \frac{\sqrt{k-1}}{k}$ when $k > 1$. We can also show that the Rayleigh quotient is lower bounded by 0 by following analogous steps in \cite{Watt16} (these steps are omitted from this appendix), and thus $\CELs(\bw)$ is convex. Next, we examine the Rayleigh quotient of the $\nabla_\bw^3\CELs(\bw)$. Analogous to the procedure for the Hessian, we make use of a 3-tensor analog of the Frobenius norm: $\| M \|_{3,F} := \sqrt{\sum_{i=1}^k\sum_{j=1}^k\sum_{l=1}^k M(i,j,l)^2}$. Thus we have the following inequality
\begin{align}
    x^T \otimes [x^\perp \otimes \nabla_\bw^3 \CELs(\bw)] \otimes x &\leq \|\nabla_\bw^3 \CELs(\bw)\|_{3,F}.
\end{align}
Since $\|z\otimes z^T \otimes z^\perp\|_{3,F} = 1$, we can write this as
\begin{equation}
    \|\nabla_\bw^3 \CELs(\bw)\|_{3,F} \leq \sqrt{
    \begin{aligned}
    & \sum_{i=1}^k (p_i - 3p_i^2 + 2p_i^3)^2  +
    \sum_{i=1}^k\sum_{j=1,j\neq i}^k (-p_ip_l + 2p_i^2p_l)^2 \\ & + \sum_{j=1}^k\sum_{l=1,l\neq j}^k (-p_jp_i + 2p_j^2p_i)^2 +  \sum_{l=1}^k\sum_{i=1,i\neq l}^k (-p_lp_j + 2p_l^2p_j)^2 \\ & + \sum_{i=1}^k\sum_{j=1,j\neq i}^k\sum_{l=1,l\neq j}^k (2p_i p_j p_l)^2.
    \end{aligned} } 
\end{equation}
This is maximized when $p_i = \frac{1}{k}$ for $k > 2$, which was verified with the symbolic integrator Mathematica \cite{Mathematica}. Simplifying results in: 
\begin{equation}
    \|\nabla_\bw^3 \CELs(\bw)\|_{3,F} \leq \sqrt{\frac{(k-1)(k-2)}{k^3}}.
\end{equation}
Thus for $\CELs(\bw)$, the smoothness constant, $c_S \leq \sqrt{\frac{(k-1)(k-2)}{k^3}}$ when $k > 2$. When $k = 2$, $p_1,p_2 = \frac{1}{2}\pm \frac{\sqrt{3}}{6}$ and $c_S \leq \sqrt{\frac{2}{27}}$.
\subsection{Study on Base-learning Learning Rate and Number of Update Steps}
In this section we present additional results on the performance of the algorithms with different iterations and learning rates using the same example setup as in Section \ref{example_circleclass}. Note that we have not used the sample convergence bound (see Appendix \ref{ap:sample convergence bound}) and present results for a single sample $\theta \sim \Thetaq$. The true values of the upper bounds for MLAP-M \cite{Amit18}, MR-MAML \cite{Yin20}, and PAC-BUS (our method) are unlikely to change by more than 5\% as the sample complexity bound does not loosen the guarantee very much. We present these results to provide a qualitative sense of the guarantees and their trends for varying base-learning learning rates and number of update steps.

Below we present test losses for MAML \cite{Finn17} (as a baseline), MLAP-M \cite{Amit18}, MR-MAML \cite{Yin20}, and PAC-BUS for base-learning rates (lr$_b$) of $0.01$ to $10$ using $\{1,3,10\}$ adaptation steps. 
\begin{table*}[ht]
\setlength\tabcolsep{3pt}
\scriptsize
\begin{center}
\begin{tabular}{lccccccccccccc}
\toprule
MAML Test Loss, lr$_b$ =  & 0.01 & 0.03 & 0.1 & 0.3 & 1 & 3 & 10 \\
\midrule
Adaptation steps = 1  & 0.184$\pm$0.007 & 0.184$\pm$0.008 & 0.168$\pm$0.006 & 0.152$\pm$0.004 & 0.120$\pm$0.001 & 0.114$\pm$0.001 & 0.133$\pm$0.007 \\ 
Adaptation steps = 3  & 0.177$\pm$0.006 & 0.179$\pm$0.002 & 0.149$\pm$0.002 & 0.126$\pm$0.001 & 0.115$\pm$0.001 & 0.106$\pm$0.001 & 0.123$\pm$0.004 \\
Adaptation steps = 10 & 0.179$\pm$0.004 & 0.155$\pm$0.002 & 0.128$\pm$0.001 & 0.124$\pm$0.001 & 0.113$\pm$0.001 & 0.104$\pm$0.002 & 0.129$\pm$0.008 \\ 
\bottomrule 
\\
\toprule
MLAP-M Test Loss, lr$_b$ = & 0.01 & 0.03 & 0.1 & 0.3  & 1 & 3 & 10 \\
\midrule
Adaptation steps = 1 & 0.181$\pm$0.010 & 0.175$\pm$0.014 & 0.150$\pm$0.006 & 0.129$\pm$0.009 & 0.083$\pm$0.001 & 0.065$\pm$0.003 & 0.220$\pm$0.044 \\ 
Adaptation steps = 3 & 0.178$\pm$0.006 & 0.159$\pm$0.007 & 0.102$\pm$0.005 & 0.081$\pm$0.003 & 0.064$\pm$0.001 & 0.050$\pm$0.004 & 0.379$\pm$0.021  \\
Adaptation steps = 10 & 0.161$\pm$0.005 & 0.115$\pm$0.002 & 0.078$\pm$0.004 & 0.063$\pm$0.002 & 0.050$\pm$0.001 & 0.045$\pm$0.002 & 0.919$\pm$0.036  \\
\bottomrule
\\ 
\toprule
MR-MAML Test Loss, lr$_b$ = & 0.01 & 0.03 & 0.1 & 0.3  & 1 & 3 & 10 \\
\midrule
Adaptation steps = 1 & 0.171$\pm$0.003 & 0.169$\pm$0.003 & 0.163$\pm$0.003 & 0.146$\pm$0.002 & 0.127$\pm$0.001 & 0.128$\pm$0.000 & 0.178$\pm$0.008  \\ 
Adaptation steps = 3 & 0.170$\pm$0.002 & 0.166$\pm$0.001 & 0.146$\pm$0.002 & 0.128$\pm$0.001 & 0.123$\pm$0.001 & 0.118$\pm$0.001 & 0.163$\pm$0.022 \\
Adaptation steps = 10 & 0.165$\pm$0.002 & 0.152$\pm$0.002 & 0.129$\pm$0.001 & 0.126$\pm$0.001 & 0.118$\pm$0.001 & 0.115$\pm$0.001 & 0.139$\pm$0.009  \\
\bottomrule
\\
\toprule
PAC-BUS Test Loss, lr$_b$ = & 0.01 & 0.03 & 0.1 & 0.3 & 1 & 3 & 10 \\
\midrule
Adaptation steps = 1  & 0.176$\pm$0.002 & 0.171$\pm$0.003 & 0.160$\pm$0.001 & 0.145$\pm$0.002 & 0.127$\pm$0.001 & 0.129$\pm$0.002 & 0.164$\pm$0.019 \\
Adaptation steps = 3 & 0.170$\pm$0.002 & 0.165$\pm$0.002 & 0.145$\pm$0.001 & 0.129$\pm$0.002 & 0.123$\pm$0.001 & 0.120$\pm$0.001 & 0.144$\pm$0.014 \\
Adaptation steps = 10 &0.163$\pm$0.001 & 0.150$\pm$0.001 & 0.130$\pm$0.001 & 0.126$\pm$0.002 & 0.119$\pm$0.002 & 0.115$\pm$0.002 & 0.130$\pm$0.004  \\
\bottomrule
\end{tabular}
\end{center}
\end{table*} 

Next, we present the computed bounds for MLAP-M \cite{Amit18}, MR-MAML \cite{Yin20}, and PAC-BUS for the same set of hyper-parameters.

\begin{table*}[ht]
\setlength\tabcolsep{3pt}
\scriptsize
\begin{center}
\begin{tabular}{lccccccccccccc}
\toprule
MLAP-M Bound, lr$_b$ = & 0.01 & 0.03 & 0.1 & 0.3 & 1 & 3 & 10 \\
\midrule
Adaptation steps = 1 & \textbf{1.003$\pm$0.000} & 1.015$\pm$0.001 & 1.223$\pm$0.020 & 1.946$\pm$0.043 & 3.113$\pm$0.154 & 5.435$\pm$0.220 & 21.874$\pm$0.420 \\
Adaptation steps = 3 & 1.008$\pm$0.000 & 1.087$\pm$0.027 & 1.864$\pm$0.062 & 3.072$\pm$0.157 & 4.147$\pm$0.095 & 6.760$\pm$0.233 & 28.356$\pm$1.826 \\  
Adaptation steps = 10 & 1.050$\pm$0.006 & 1.535$\pm$0.044 & 2.574$\pm$0.064 & 4.009$\pm$0.107 & 5.98$\pm$0.057 & 10.119$\pm$0.087 & 47.971$\pm$1.346  \\
\bottomrule
\\
\toprule
MR-MAML Bound, lr$_b$ = & 0.01 & 0.03 & 0.1 & 0.3 & 1 & 3 & 10 \\
\midrule
Adaptation steps = 1 & 0.344$\pm$0.002 & 0.343$\pm$0.002 & 0.335$\pm$0.002 & 0.320$\pm$0.001 & 0.300$\pm$0.000 & 0.303$\pm$0.001 & 0.351$\pm$0.006 \\ 
Adaptation steps = 3 & 0.344$\pm$0.002 & 0.340$\pm$0.002 & 0.320$\pm$0.002 & 0.302$\pm$0.002 & 0.296$\pm$0.001 & \textbf{0.292$\pm$0.001} & 0.335$\pm$0.018  \\
Adaptation steps = 10 & 0.339$\pm$0.001 & 0.324$\pm$0.002 & 0.303$\pm$0.000 & 4.752$\pm$0.808 & 5.330$\pm$0.187 & 6.316$\pm$0.639 & 9.134$\pm$1.448 \\
\bottomrule
\\
\toprule
PAC-BUS Bound, lr$_b$ = & 0.01 & 0.03 & 0.1 & 0.3 & 1 & 3 & 10 \\
\midrule
Adaptation steps = 1 & 0.216$\pm$0.002 & 0.216$\pm$0.002 & 0.204$\pm$0.002 & 0.188$\pm$0.002 & \textbf{0.169$\pm$0.000} & 0.171$\pm$0.001 & 0.207$\pm$0.021  \\
Adaptation steps = 3 & 0.252$\pm$0.001 & 0.247$\pm$0.002 & 0.228$\pm$0.002 & 0.211$\pm$0.001 & 0.204$\pm$0.001 & 0.200$\pm$0.002 & 0.228$\pm$0.017 \\
Adaptation steps = 10 & 0.383$\pm$0.002 & 0.372$\pm$0.001 & 0.350$\pm$0.001 & 1.160$\pm$0.093 & 1.288$\pm$0.056 & 1.650$\pm$0.055 & 2.221$\pm$0.256 \\
\bottomrule
\end{tabular}
\end{center}
\end{table*} 
These results show the dependence that the PAC-BUS upper bound (specifically the uniform stability regularizer term $\b$) has on the learning rate and number of base-learning update steps whereas the bound for MR-MAML does not suffer with increasing base-learning steps or learning rate. However, once the learning rate and number of adaptation steps are too large, all bounds worsen significantly. The tightest guarantee obtained using PAC-BUS is significantly stronger than those for any tuning of MR-MAML and MLAP-M. We bold the tightest guarantee achieved in the tables above to highlight this.

\subsection{Additional Experimental Details}
\label{sec:experimental details}
\newcommand{\expnumber}[2]{{#1}\mathrm{e}{#2}}

In this section, we report information about the data used, the procedure for prior, train, and test splits, as well as other experimental details. 
% Code capable of reproducing the results in this paper is attached as supplementary material. 
Code capable of reproducing the results in this paper is publicly available at \url{https://github.com/irom-lab/PAC-BUS}.
All results provided in this paper were computed on an Amazon Web Services (AWS) p2 instances. Tuning and intermediate results were computed on a desktop computer with a $12$-core Intel i7-8700k CPU and an NVIDIA Titan Xp GPU. In addition, we made use of several existing software assets:  SciKit-learn \cite{scikit-learn} (BSD license), PyTorch \cite{pytorch} (BSD license), CVXPY \cite{cvxpy1, cvxpy2} (Apache License, Version 2.0), MOSEK \cite{mosek} (software was used with a personal academic license, see \url{https://www.mosek.com/products/license-agreement} for more details), learn2learn \cite{learn2learn} (MIT License), and h5py \cite{h5py} (Python license, see \url{https://docs.h5py.org/en/stable/licenses.html} for more details).

\subsubsection{Circle Class} \label{expdets_circleclass}
We randomly sample points from the unit ball $B^2(0,1)$ and classify them as $(+)$ or $(-)$ according to whether or not the points are outside the ball $B^2(c_t,r_t)$. For the tasks which are used to train a prior, we sample $c_t$ from $[0.1, 0.5]$ and $r_t$ from $[0.1, 1-\|c_t\|]$. For the meta-training and meta-testing tasks, we sample $c_t$ from $[0.1, 0.4]$ and $r_t$ from $[0.1, 1-\|c_t\|]$. 

For all methods, we train the prior on 500 tasks, train the network on 10000 tasks, and test on 1000 tasks. We report the meta-test loss and a guarantee on the loss if applicable. A single task is a $2$-class $10$-sample (i.e. there are $10$ samples given in total for training, not $10$ samples from each class) learning problem. The evaluation dataset $S_\ev$ consists of a dataset $S$ of $10$ base-learner training samples and a dataset $S_\va$ of $250$ validation samples. For PAC-BUS, we searched for the meta-learning rate in $[\expnumber{1}{-4}, 1]$, the base-learning rate in $[0.01, 10]$, and the number of base-learning update steps in $[1,10]$. The resulting parameters for the $10$-shot learning problems are: meta-learning rate $\expnumber{1}{-3}$, base-learning rate $0.05$, and $1$ base-learning update step. Note that in this example and the \textit{Mini-wiki} example, we select the number of base-learning steps such that the upper bound is minimized. A lower loss may have been achievable with more base-learning update steps, but we aim to produce the tightest bound possible. Training for each method took less than 1 hour on the AWS p2 instance and computing the sample convergence upper bound took approximately 3 days when applicable.

\vspace{-3pt}
\subsubsection{Mini-wiki}
\vspace{-3pt}
\label{ap:expdets_miniwiki}

In Table \ref{miniwiki results score}, we present additional results -- the percentage of correctly classified sentences on test tasks (after the base learner's adaptation step). Note that we present these results with the same posterior as was used to generate the results in Table \ref{classification results}.

\begin{table}[h]
% \vskip -0.10in
\caption{Meta-test accuracy as a percentage for MAML, FLI-Batch, MR-MAML, and PAC-BUS. We report the mean and standard deviation after 5 trials.} \label{miniwiki results score}
\begin{center}
\begin{small}
\begin{sc}
\begin{tabular}{lcccr}
\toprule
\hspace{-0.11in} $4$-Way \emph{Mini-Wiki} & 1-shot $\uparrow$ & 3-shot $\uparrow$ & 5-shot $\uparrow$ \hspace{-0.14in} \\
\midrule
 \hspace{-0.11in} MAML \cite{Finn17} & \textBF{$60.2\pm0.9$} & $68.3\pm0.7$ & \textBF{$71.9\pm0.6$} \hspace{-0.14in} \\
\hspace{-0.11in} FLI-Batch \cite{Khodak19} & $46.0\pm5.9$ & $48.7\pm4.9$ & $54.5\pm2.4$ \hspace{-0.14in} \\
\hspace{-0.11in} MR-MAML \cite{Yin20} & $59.9\pm0.8$ & \textBF{$68.4\pm0.7$} & $71.8\pm0.7$ \hspace{-0.14in} \\
\hspace{-0.11in} PAC-BUS (ours) & $59.9\pm0.8$ & $68.1\pm0.7$ &$71.2\pm0.7$  \hspace{-0.14in} \\
\bottomrule
\end{tabular}
\end{sc}
\end{small}
\end{center}
\vskip -0.2in
\end{table}

We use the \textit{Mini-wiki} dataset from \cite{Khodak19}, which consists of 813 classes each with at least 1000 example sentences from that class's corresponding Wikipedia article. The dataset was derived from the Wiki3029 dataset presented in \cite{Arora19}, which was created from a public domain (CC0 license) Wikipedia dump.
Although the Wikipedia dump is open source, it is possible that content which is copyrighted was used since the datasets are large and it is difficult to moderate all content on the website. %We are unable to check for copyrighted material 
In addition, it is possible that the dataset has some offensive content such as derogatory terms or curse words. However, since these are in the context Wikipedia articles, the authors trust that the original article was not written maliciously, but for the purposes of education. We use the first 62 classes of \textit{Mini-wiki} for training the prior, the next 625 for the meta-training, and the last 126 for meta-testing. Before creating learning tasks, we remove all sentences with fewer than 120 characters. 

For all methods, we train the prior on 100 tasks, train the network on 1000 tasks, and test on 200 tasks. We report the meta-test score, the meta-test loss, and a guarantee on the loss if applicable. A single task is a $4$-class $\{1,3,5\}$-shot learning problem. The evaluation dataset $S_\ev$ consists of a dataset $S$ of $\{1,3,5\}$ base-learner training samples and a dataset $S_\va$ of $\{250,250,250\}$ validation samples respectively. For PAC-BUS, we search for the meta-learning rate in $[0.01, 1]$ the base-learning rate in $[\expnumber{1}{-3}, 100]$, and the number of base-learning update steps in $[1,50]$. The resulting parameters for the $\{1,3,5\}$-shot learning problems are: meta-learning rate $\{0.1, 0.1, 0.1\}$, base-learning rate $\{2.5, 5, 5\}$, and $\{2,4,5\}$ base-learning update steps respectively. Training for each method took less than 1 hour on the AWS p2 instance and computing the sample convergence upper bound took approximately 2 days when applicable.

\subsubsection{Omniglot}
\vspace{-3pt}
\label{expdets_omniglot}

We use the \textit{Omniglot} dataset from \cite{Lake11}, which consists of 1623 characters each with 20 examples. The dataset was collected using Amazon's Mechanical Turk (AMT) and is available on GitHub with an MIT license. This dataset was collected voluntarily by AMT workers. Since the dataset is small enough, it can be checked visually for personally-identifiable information.
We use the first 1200 characters for meta-training and the remaining 423 for meta-testing. The image resolution is reduced to $28 \times 28$. 
In the non-mutually exclusive setting, the 1200 training characters are randomly partitioned into 20 equal-sized groups which are assigned a fixed class label from $1$ to $20$. Note that this is distinct from the method described in \cite{Yin20} where the data is partitioned into 60 disjoint sets. Both experimental setups cause memorization, but the setup used in \cite{Yin20} causes more severe memorization than ours. This is why our implementation of MAML performs better than the results for MAML reported in \cite{Yin20}. However, our implementation of MR-MAML(W) method performs similarly to what is reported in \cite{Yin20}. 

For all methods, we trained on 100000 batches of 16 tasks and report the meta-test score on 8000 test tasks. We also used $5$ base-learning update steps for all methods. A single task is a $20$-way $\{1,5\}$-shot learning problem. The evaluation dataset $S_\ev$ consists of a dataset $S$ of $\{1,5\}$ base-learner training samples and a dataset $S_\va$ of $\{4, 5\}$ validation samples respectively. For PAC-BUS(H), we searched for the regularization scales $\lambda_1$ and $\lambda_2$ in $[\expnumber{1}{-7}, 1]$ and $[\expnumber{1}{-4}, \expnumber{1}{4}]$ respectively. Additionally, the meta-learning rate was selected from $[\expnumber{5}{-4}, 0.1]$, and the base-learning rate was selected from $[0.01, 10]$. The resulting parameters for the $\{1,5\}$-shot learning problems are: $\lambda_1=\{\expnumber{1}{-3}, \expnumber{1}{-4}\}$, $\lambda_2=\{10,10\}$, meta-learning rate $\{\expnumber{1}{-3},\expnumber{1}{-3}\}$, and base-learning rate $\{0.5, 0.5\}$ respectively. Training for each method took approximately 3 days on the AWS p2 instance.

\end{document}